\documentclass{article}

\usepackage{arxiv}
\usepackage{amsmath,amssymb}
\usepackage{amsthm}
\usepackage[utf8]{inputenc} 
\usepackage[T1]{fontenc}   
\usepackage{hyperref}       
\usepackage{url}            
\usepackage{booktabs}       
\usepackage{amsfonts}       
\usepackage{nicefrac}       
\usepackage{microtype}    
\usepackage{graphicx}
\usepackage{algorithm}
\usepackage{algorithmicx}
\usepackage{algpseudocode}
\usepackage{subfig}
\graphicspath{ {./images/} }
\usepackage{mathdots}
\usepackage{float}

\usepackage{yhmath}
\usepackage{tikz}
\usetikzlibrary{fadings}
\usetikzlibrary{patterns}
\usetikzlibrary{shadows.blur}
\usetikzlibrary{shapes}

\newtheorem{theorem}{Theorem}

\newtheorem{lemma}{Lemma}
\newtheorem{corollary}{Corollary}
\newtheorem{proposition}{Proposition}
\newtheorem{definition}{Definition}
\def\map#1#2{\def\mapname{#1}\def\mapdomain{#2}\def\maparrow{\longrightarrow}\futurelet\next\mapAR}
\def\mapAR{\ifx\next[\expandafter\mapar\else\expandafter\mapAux\fi}
\def\mapar[#1]{\def\maparrow{#1}\expandafter\mapAux}
\def\mapAux#1#2{\def\mapcodomain{#1}\def\mapelement{#2}\def\mapmaparrow{\longmapsto}\futurelet\next\mapMAR}
\def\mapMAR{\ifx\next[\expandafter\mapmar\else\expandafter\mapAuxAux\fi}
\def\mapmar[#1]{\def\mapmaparrow{#1}\expandafter\mapAuxAux}
\def\mapAuxAux#1{\def\mapcoelemnt{#1}\futurelet\next\mapAuxAuxAux}
\def\mapAuxAuxAux{\def\nothing{}
\ifx\mapname\nothing
  \mapdomain\maparrow\mapcodomain\qquad\mapelement\mapmaparrow\mapcoelemnt
\else
  \mapname:\mapdomain\maparrow\mapcodomain\qquad\mapelement\mapmaparrow\mapcoelemnt
\fi
}
\def\Map#1#2{\def\Mapname{#1}\def\Mapdomain{#2}\def\Maparrow{\longrightarrow}\futurelet\next\MapAR}
\def\MapAR{\ifx\next[\expandafter\Mapar\else\expandafter\MapAux\fi}
\def\Mapar[#1]{\def\Maparrow{#1}\expandafter\MapAux}
\def\MapAux#1#2{\def\Mapcodomain{#1}\def\Mapelement{#2}\def\Mapmaparrow{\longmapsto}\futurelet\next\MapMAR}
\def\MapMAR{\ifx\next[\expandafter\Mapmar\else\expandafter\MapAuxAux\fi}
\def\Mapmar[#1]{\def\Mapmaparrow{#1}\expandafter\MapAuxAux}
\def\MapAuxAux#1{\def\Mapcoelemnt{#1}\futurelet\next\MapAuxAuxAux}
\def\MapAuxAuxAux{\def\nothing{}%
\ifx\Mapname\nothing%
  \begin{array}{rcl}{}
  \Mapdomain&\Maparrow&\Mapcodomain\\{}
  \Mapelement&\Mapmaparrow&\Mapcoelemnt
  \end{array}
\else%
  \begin{array}{rrcl}{}
  \Mapname:&\Mapdomain&\Maparrow&\Mapcodomain\\{}
  &\Mapelement&\Mapmaparrow&\Mapcoelemnt
  \end{array}
\fi%
}

\def\p{{\rm \partial}}

\title{Manifold Regularization Classification Model Based On Improved Diffusion Map}

\author{
 Hongfu Guo \\
  Leceister Institution\\
  Dalian University of Technology\\
  \texttt{hg203@student.le.ac.uk} \\
   \And
 Wencheng Zou \\
 Dalian Leceister Institution\\
  Dalian University of Technology\\
  \texttt{wz159@student.le.ac.uk} \\
  \And
 Zeyu Zhang \\
  Dalian Leceister Institution\\
  Dalian University of Technology\\
  \texttt{zz263@student.le.ac.uk} \\
  \And
  Shuishan Zhang\\
  Dalian Leceister Institution\\
  Dalian University of Technology\\
  \texttt{sz252@student.le.ac.uk}\\
  \And
  Ruitong Wang\\
  Dalian Leceister Institution\\
  Dalian University of Technology\\
  \texttt{rw341@student.le.ac.uk}\\
  \And
    Jintao Zhang\\
    Dalian Leceister Institution\\
    Dalian University of Technology\\
    \texttt{zjt@dlut.edu.cn}\\
}

\begin{document}
\maketitle
\begin{abstract}
   Manifold regularization model is a semi-supervised learning model that leverages the geometric structure of a dataset, comprising a small number of labeled samples and a large number of unlabeled samples, to generate classifiers. However, the original manifold norm limits the model’s performance to local regions. To address this limitation, this paper proposes an approach to improve manifold regularization based on a label propagation model. We initially enhance the probability transition matrix of the diffusion map algorithm, which can be used to estimate the Neumann heat kernel, enabling it to accurately depict the label propagation process on the manifold. Using this matrix, we establish a label propagation function on the dataset to describe the distribution of labels at different time steps. Subsequently, we extend the label propagation function to the entire data manifold. We prove that the extended label propagation function converges to a stable distribution after a sufficiently long time and can be considered as a classifier. Building upon this concept, we propose a viable improvement to the manifold regularization model and validate its superiority through experiments.
\end{abstract}


\section{Introduction}
\subsection{An overview of Semi-Supervised learning}
Semi-supervised learning classification algorithms are a class of algorithms that utilize a small amount of labeled data and a large amount of unlabeled data for classification tasks. Compared to supervised learning algorithms that only use labeled data, semi-supervised learning algorithms can fully utilize the information from unlabeled data, thereby improving classification performance. Classic semi-supervised learning classification algorithms include Semi-Supervised Support Vector Machines (S3VM), Self-Training algorithms, Generative Classification Models, and Label Propagation Algorithms. Below, we provide an overview of these algorithms.

    {\bf Semi-Supervised Support Vector Machines(S3VM)} is based on the principles of traditional Support Vector Machines (SVM), aiming to find a hyperplane that separates data from different classes while maintaining the maximum margin possible. Unlike traditional SVM, S3VM incorporates unlabeled data to fully utilize this additional information(See \cite{bennett1998semi}). In the optimization objective function, S3VM minimizes misclassification of labeled data and boundary violations of unlabeled data. The goal is to maintain the accuracy of labeled data classification while leveraging the information from unlabeled data to improve classification performance. However, S3VM still suffers from assumptions about unlabeled data and potential issues such as local optima.
    
      {\bf The Self-Training algorithm} involves training an initial classifier using labeled data, then using this classifier to predict unlabeled data, selecting samples with high prediction confidence as pseudo-labels, adding them to the labeled data, and retraining the classifier(See \cite{yarowsky1995unsupervised}). This process is repeated until convergence or reaching a certain number of iterations. The advantage of the Self-Training algorithm is its simplicity, but it is susceptible to the influence of the initial classifier and pseudo-labels, leading to error accumulation.
      
       {\bf  the Generative Classification Model} is to assume that both labeled and unlabeled data come from the same underlying data distribution. Labeled data is used to estimate the parameters of this distribution, which are then used to generate new data or classify unlabeled data. Generative Adversarial Networks (GANs) are a common type of generative model that utilizes adversarial competition to improve classification performance(See \cite{ng2001discriminative}). The advantage of generative models is their ability to utilize statistical information from unlabeled data, but the assumption about the data distribution may not always hold true.
       
       {\bf The graph-based Label Propagation Algorithm} constructs a graph with all data (labeled and unlabeled) as nodes and edges representing the similarity or distance between samples. Then, based on the label information on the graph, label propagation or random walks are performed to assign appropriate labels to unlabeled nodes(See \cite{zhu2002learning}). The advantage of graph-based algorithms is their ability to fully utilize structural information between data, but they suffer from high computational complexity, requiring large amounts of memory and time.
\subsection{Outline of Paper}
In this paper, we focus on the manifold regularization classification model improved and optimized based on the label propagation algorithm. Specifically, we innovatively use the diffusion mapping algorithm improved based on geodesic distance to realize label propagation of the dataset. Finally, the model exhibits good classification performance. The specific structure of this paper is as follows:

In the first part, we introduce the original manifold regularization model and its shortcomings in sample labeling.

In the second part, we elaborate on the two-step improvement of our manifold regularization model. The first step is to construct a label propagation model, where we improve the classical diffusion mapping algorithm based on geodesic distance. The second step is to build an appropriate penalty term, where we improve the penalty term of the original manifold regularization model loss function based on the Neumann Heat Kernel operator. Ultimately, we successfully construct the Neumann Heat Kernel manifold regularization model (NHKMR).

 In the third part, we conduct three types of numerical experiments.
  \begin{itemize}
      \item   Experiment 1: We use the NHKMR model and Manifold regularization model based on the Laplace operator (LapMR) to classify three datasets: twoMoons, Ring, and twoClusters.
   \item Experiment 2.1: We compare the error rates of three types of models—Neumann Heat Kernel Regularized Least Squares (NHKRLS), Regularized Least Squares model based on the Laplace operator (LapRLS), and classical Least Squares (LS)—in handling binary classification tasks.
   \item Experiment 2.2: We compare the classification effects of NHKRLS model, LapRLS model, and LS model on multi-classification tasks using handwritten digit datasets.
  \end{itemize}

 In the fourth part, we conclude that based on the label propagation model, we propose a manifold regularization classification model with stronger feature extraction capabilities for sample points and provide the advantages and disadvantages of the model.
\section{Acknowledgement}
This article has been polished by ChatGPT-3.5.
\section{Manifold Regularization Model and its shortcomming }%
Manifold regularization is a semi-supervised classification learning method proposed by Belkin et al. in 2006. The goal of manifold regularization is to address classification problems when the number of unlabeled data points far exceeds the number of labeled data points. It is inspired by the assumption of the manifold structure in the data, suggesting that data points in high-dimensional space may be distributed along a low-dimensional manifold. On the manifold, nearby data points should have similar labels.

Based on this idea, the manifold regularization model's loss function is divided into three parts: A loss function of labeled data, a kernel regularization term and a manifold regularization term. 
$$\min_{f\in\mathcal{H}_{K}}\frac{1}{l}\sum_{i=1}^{l}V(x_{i},y_{i},f) + \gamma_{K}\left\lVert f \right\rVert _{K}^2 + \gamma_{I}\left\lVert f \right\rVert _{I}^2  $$
The second term ensures the smoothness of the classifier, and the third term is a crucial component of the model. It considers the geometric features between sample points and controls that within a small neighborhood of each sample point, the output of the classifier should be the same; otherwise, it incurs a penalty.

The choice of the manifold term is quite flexible, and a natural choice is to set it as the $L^2$ semi-positive definite inner product generated by the Laplacian operator
$$
\left\lVert f \right\rVert _{I}^2 = -\int_{\mathcal{M}}f\Delta_{\mathcal{M}}f\,d\mu=\int_{\mathcal{M}}\left\lVert \nabla_{\mathcal{M}}f \right\rVert ^2\,d\mu 
$$
This norm, defined in such a way, has been proven by Belkin and Niyogi (2003) to be a continuous form of the graph Laplacian operator(See \cite{belkin2001laplacian}). Furthermore, the representation theorem states that as long as $D$ is a bounded linear operator, and the manifold term is defined as $\left\langle f,Df \right\rangle _{L^2}$, the optimal solution of this model can be expressed as
$$f^{*}=\sum_{i=1}^{l+u}\alpha_{i}K(x_{i},\cdot)$$
where $l$ is the number of labeled data and $u$ is the number of unlabeled data(See \cite{Belkin}). 

 Belkin also proposed the idea of using the heat kernel operator to define the manifold term but did not further discuss it(See \cite{Belkin}). In the following section, we will take this perspective and introduce a definition of the manifold term using the Neumann heat kernel operator. 

 Although manifold regularization has been applied in various fields, the relationship between the number of labeled and unlabeled samples and model error remains a significant concern(See \cite{belkin2004regularization}). Intuitively, when the number of labeled data points is limited, the geometric structure of data points has a substantial impact on the classification plane. According to the smoothness assumption, samples closer to labeled data points are expected to yield more accurate classification labels. If the number of labeled data is insufficient, the model can only achieve good local classification performance. To enhance the global classification capability of the model, a direct approach is to increase the number of labeled samples. Additionally, we can utilize manifold terms that make the geometric structure of sample points more pronounced. We will improve the model based on these two ideas.

\section{Manifold Regularization Model and its shortcomming }%
Manifold regularization is a semi-supervised classification learning method proposed by Belkin et al. in 2006. The goal of manifold regularization is to address classification problems when the number of unlabeled data points far exceeds the number of labeled data points. It is inspired by the assumption of the manifold structure in the data, suggesting that data points in high-dimensional space may be distributed along a low-dimensional manifold. On the manifold, nearby data points should have similar labels.

Based on this idea, the manifold regularization model's loss function is divided into three parts: A loss function of labeled data, a kernel regularization term and a manifold regularization term. 
$$\min_{f\in\mathcal{H}_{K}}\frac{1}{l}\sum_{i=1}^{l}V(x_{i},y_{i},f) + \gamma_{K}\left\lVert f \right\rVert _{K}^2 + \gamma_{I}\left\lVert f \right\rVert _{I}^2  $$
The second term ensures the smoothness of the classifier, and the third term is a crucial component of the model. It considers the geometric features between sample points and controls that within a small neighborhood of each sample point, the output of the classifier should be the same; otherwise, it incurs a penalty.

The choice of the manifold term is quite flexible, and a natural choice is to set it as the $L^2$ semi-positive definite inner product generated by the Laplacian operator
$$
\left\lVert f \right\rVert _{I}^2 = -\int_{\mathcal{M}}f\Delta_{\mathcal{M}}f\,d\mu=\int_{\mathcal{M}}\left\lVert \nabla_{\mathcal{M}}f \right\rVert ^2\,d\mu 
$$
This norm, defined in such a way, has been proven by Belkin and Niyogi (2003) to be a continuous form of the graph Laplacian operator(See \cite{belkin2001laplacian}). Furthermore, the representation theorem states that as long as $D$ is a bounded linear operator, and the manifold term is defined as $\left\langle f,Df \right\rangle _{L^2}$, the optimal solution of this model can be expressed as
$$f^{*}=\sum_{i=1}^{l+u}\alpha_{i}K(x_{i},\cdot)$$
where $l$ is the number of labeled data and $u$ is the number of unlabeled data(See \cite{Belkin}).

 Belkin also proposed the idea of using the heat kernel operator to define the manifold term but did not further discuss it(See \cite{Belkin}). In the following section, we will take this perspective and introduce a definition of the manifold term using the Neumann heat kernel operator. 

 Although manifold regularization has been applied in various fields, the relationship between the number of labeled and unlabeled samples and model error remains a significant concern(See \cite{belkin2004regularization}). Intuitively, when the number of labeled data points is limited, the geometric structure of data points has a substantial impact on the classification plane. According to the smoothness assumption, samples closer to labeled data points are expected to yield more accurate classification labels. If the number of labeled data is insufficient, the model can only achieve good local classification performance. To enhance the global classification capability of the model, a direct approach is to increase the number of labeled samples. Additionally, we can utilize manifold terms that make the geometric structure of sample points more pronounced. We will improve the model based on these two ideas.

 \section{The improved Manifold Regularization Model}%
 
Our improvement will be divided into two steps:
\begin{description}
  \item[Step 1:] Locally predict the labels of unlabeled sample points based on labeled sample points. This step requires constructing a label propagation model.
  \item[Step 2:] Build an appropriate penalty based on the label propagation model.
\end{description}
\subsection{The construction of Label propagation model}
Firstly, let's construct the label propagation model. Given a labeled sample point, we expect that the label of other sample points closest to it is most likely to be influenced by the label of the labeled sample point, indicating a higher probability of sharing the same label. This influence diminishes as the distance between sample points increases. This naturally leads us to consider a heat conduction process in space. Suppose there is a constant-temperature heat source in space, and all other points have a temperature of 0. According to the second law of thermodynamics(See \cite{feynman2011feynman}), over a short period, the heat conductor closest to the constant-temperature heat source will be influenced, its temperature will rise, and these conductors will temporarily become a "transient" heat source, transferring heat to their surroundings. The other conductors closest to them will gain heat and experience a temperature increase. However, when the time is short enough, these "transient" heat sources cannot acquire all the heat from the constant-temperature heat source in time, resulting in weaker heat transfer effects compared to the constant-temperature heat source. By repeating this process, we can obtain a heat distribution in space that decreases with distance, which is precisely what we want. Based on the above discussion, we construct the label propagation model for the binary classification problem as follows:

The dataset $X$ can be partitioned by few numbers of labeled and enormous numbers of unlabeled data $X = X_{l}\sqcup X_{u}$.  Setting the label propagation function 
  $$\map{u}{X\times \mathbb{R}_{+}}{\mathbb{R}}{(x,t)}{u(x,t)}$$
  satisfies
  \begin{itemize}
    \item The function values of the data points in the labeled dataset are given by:
    $$u(x,t)|_{X_{l}} \equiv u(x,0)|_{X_{l}}\in\{+1,-1\} \quad \forall t\in \mathbb{R}_{+}$$
    \item Initially, the function values of unlabeled data should be zero. i.e. $u(x,0)|_{X_{u}} = 0$
  \end{itemize}
  In practical applications, the input sample points used to train the model are finite, causing the label propagation function described above to lose its smoothness. According to the manifold assumption in semi-supervised learning, the high-dimensional data (such as images, text, or sound) lies on a low-dimensional manifold embedded within the high-dimensional space. Therefore, it is necessary to extend the label propagation function to this smooth sub-manifold.
  
  The the extented function $\bar{u}:\mathcal{M}\times\mathbb{R}_{+}\to \mathbb{R}$ should satisfies
  \begin{enumerate}
    \item $\bar{u}(x,\cdot)|_{X}= u(x,\cdot)$
    \item $\bar{u}(x,\cdot)\in C^{\infty}(\mathcal{M})$ 
  \end{enumerate}

  The propagation of labels should adhere to a certain pattern. Without loss of generality, setting the propagation process follows the heat equation. That is, 
  \begin{gather}
    \left\{
    \begin{array}{l}
      \p_{t}\bar{u}(x,t) = \Delta \bar{u}(x,t)\\
      \bar{u}(x,t)|_{X_{l}} = \bar{u}(x,0)|_{X_{l}} \\
      \bar{u}(x,0)|_{X_{u}} = 0
    \end{array}
  \right.
  \end{gather}
  Under the above conditions, we can view the label propagation problem as a heat conduction problem with a constant-temperature heat source(See \cite{thermal}). Thus, it can be rewritten by :
  \begin{gather}
    \left\{\begin{array}{l}
      \p_{t}\bar{u}(x,t) = \Delta\bar{u}(x,t)+Q(x)\\
      \bar{u}(x,0) = Q(x)
    \end{array}\right.
  \end{gather}
where $Q(x)$ is the distribution of label when $t=0$. That is, $Q|_{X_{l}}(x)\in\{-1,+1\}$, and $Q|_{X_{u}}(x) \equiv 0 $

  Alternatively, we consider the integral form $$\bar{u}(x,t) +w(x)= e^{t\Delta}\left(\bar{u}(x,0)+w(x)\right)$$, where $\Delta w(x)=Q(x)$ 
  
  Based on the aforementioned assumptions, we can verify that unlabeled data points near labeled data points will acquire the affinity labels as the labeled ones after undergoing a certain diffusion process over time.

  Here we must ensure the existence of extension. It suffices to show that the existence of propagation function when $t=0$. In fact, we have following conclusion
\begin{theorem}[The existence of initial propagation function]
  $X$ is a dataset with a finite number of elements. Considering the binary classification problem, $X_{l}$ is the labeled dataset, and $X_{u}$ is the unlabeled dataset. There exists $Q(x)\in C^{\infty}(\mathcal{M})$ such that
   $$\mathrm{range}(Q|_{X_{l}}) = \{+1,-1\},\quad \mathrm{range}(Q|_{X_{u}})=\{0\}$$  
   where $\mathcal{M}$ is a sub-manifold the dataset embedded.
\end{theorem}
The proof of this theorem will be demonstrated in Appendix B.

Unfortunately, the heat kernel operator on the manifold does not have an explicit expression. As an alternative, we will use another operator which can be obtained numerically to estimate the heat kernel operator and iteratively provide the values of the Label Propagation function at different time steps.

There are various methods to solve the heat conduction problem on manifolds, with the most mainstream ones being finite element methods, Beltrami methods, and spectral methods. However, these approaches typically require prior knowledge of the manifold's structure. If the structure of the sub-manifold where the dataset resides is unknown, we are forced to rely on the dataset itself to estimate the heat kernel operator. In fact, we have the following theorem(See \cite{Lafon}).
  \begin{theorem}
   Suppose that $P_{\varepsilon}$ is the transition operator defined by the Gaussian kernel $k_{\varepsilon}(x,y) = \exp\left(\frac{\left\lVert x-y \right\rVert ^2 }{\varepsilon}\right)$, then it satisfies
   $$P^{\frac{t}{\varepsilon}}_{\varepsilon}\underset{{L^2(\mathcal{M})}}{\longrightarrow} e^{t\Delta}\quad as \quad \varepsilon\to 0$$
    Here $\Delta$ is the Laplacian-Beltrami operator on $\mathcal{M}$. In the geodesic coordinate system $(s_{1},\cdots,s_{d})$, it is expressed as 
    $\Delta f = \sum_{i=1}^{d} \frac{\p^2 f}{\p s_{i}^2}$
    \label{thm1}
  \end{theorem}
  The matrix $P_{\varepsilon}$ is the diffusion mapping matrix generated from the dataset, and this theorem is based on the diffusion mapping algorithm proposed by Lafon et al. When $\varepsilon$ is sufficiently small, the diffusion mapping matrix $L^2$ converges to the Neumann Heat Kernel. This demonstrates that the algorithm can numerically simulate the heat diffusion process on the dataset. We can use this algorithm to simulate label propagation on the dataset.

  \subsection{Classical diffusion map}
Diffusion Mapping is a non-linear dimensionality reduction method. Extract desired features from a sample set, requires constructing a connectivity graph for the samples with different weights assigned to each edge. In contrast to manifold learning algorithms like ISOMAP and LE, Diffusion Mapping does not directly search for suitable geometric features on the connectivity graph. Instead, inspired by diffusion processes in random dynamical systems, it constructs transition probabilities between different points using a symmetric positive-definite kernel function. These data points are then placed in a Markov matrix. By performing a spectral decomposition on the Markov matrix, arranging eigenvalues in descending order, and selecting a limited number of corresponding eigenvectors, it establishes a mapping from high-dimensional space to low-dimensional space, thereby achieving dimensionality reduction.

The diffusion process refers to the movement of particles from an area of higher concentration to an area of lower concentration, resulting in the even distribution of the particles throughout the medium they are in. This movement is driven by the random thermal motion of particles, known as Brownian motion. The spread of particles like the energy transition, which indicates that it can be described by the heat equation. 

\begin{definition}[Kernel function]
  Suppose that $(X,\mu,\mathcal{A})$ is a measure space with the collection of data points $X$, the distribution of data $\mu$ and the $\sigma$-algebra $\mathcal{A}$ of $X$. $k: X\times X\to \mathbb{R}$ is a kernel function provided that it satisfies 
   \begin{enumerate}
    \item Symmetric: $k(x,y) = k(y,x) \quad \forall x,y\in X$
    \item Positive definite: $k(x,x)\geqslant 0\quad \forall x\in X$
   \end{enumerate}
\end{definition}

Kernel functions are commonly used to measure the similarity between samples. Although the kernel function for constructing diffusion maps only needs to be isotropic, in this paper, we uniformly choose the Gaussian kernel function for our study for simplicity. 

We denote $d(x)$ as the cumulative similarity, which is defined by
$$d(x) = \int_{X}k(x,y)\,dy$$
Let 
$$p(x,y)=\frac{k(x,y)}{d(x)}$$
be the transition probability from $x$ to $y$. Although it is non-symmetric, its conservation property indicates that a transition operator $P$ can be defined by 
$$Pf(x) = \int_{X}p(x,y)f(y)\,d\mu(y)$$
This operator maps a density distribution function $f$ in space into the density distribution function after undergoing a diffusion process. When we choose an appropriate probability measure, this operator can be represented by a matrix: $P = (p(x,y))_{x,y\in X}$, which can be used to update label propagation function $u(x,t)$.

Here is the framework of the update algorithm( {\bf Algorithm \ref{Alg 1}}). Let $u(x,t)$ be the label propagation function at data point $x$, time $t$. The value of this function at $t=0$ was $u(x,0)|_{X_{l}} \in\{+1,-1\}$, $u(x,0)|_{X_{u}} \equiv 0 $. According to {\bf Theorem \ref{thm1}},  substituting the heat kernel to the diffusion matrix, propagation function updates could be estimated by $u(x,t) = P^{\frac{t}{\varepsilon}}_{\varepsilon}u(x,0)$. Although we had not defined exponentiation for the operator $P$ in the general case, when $X$ was a finite set, the operator was defined on finite-dimensional spaces and followed the rules of matrix exponentiation, and $u(x,\cdot)$ is also a vector in this case. In , we have given the update method of $u(x,t)$.
\begin{algorithm}[htbp]
 \begin{algorithmic}
      \State{Input dataset $X_{l}$, $X_{u}$, Diffusion steps $t\in\mathbb{Z}^{+}$, kernel function $k_{\varepsilon}(x,y)$}
  \State{Update kernel matrix $K = (k_{\varepsilon}(x,y))_{x,y\in X}$}
  \State{Update $(d_{\varepsilon}(x))_{x\in X}$, $d_{\varepsilon}(x) = \sum_{y\in X}k(x,y)$}
  \State{Update probability transition matrix $P = (p(x,y))_{x,y\in X}$, $p(x,y) = \frac{k_{\varepsilon}(x,y)}{d_{\varepsilon}(x)}$}
  \State{Initialise the label propagation function $u|_{X_{l}}\in\{-1,+1\}$, $u|_{X_{u}}=0$}
  \For{k = 0 $\cdots$ t}
     \State{$u = P_{\varepsilon}^{\frac{1}{\varepsilon}} u$}
     \State{$u|_{X_{l}}(x,t) = u|_{X_{l}}(x,0)$}
    \EndFor
 \end{algorithmic}
  \caption{Update Label Propagation function $u(x,t)$}
  \label{Alg 1}
  \end{algorithm}
  
\subsubsection{The shortcomings of classical diffusion map models }
    Although the probability transition matrix can approximate the heat kernel operator, it may not correctly capture the feature of label propagation process on the dataset.

    The diffusion process on the helix-shaped dataset is illustrated in {\bf Figure \ref{fig:spiral2},\ref{fig:spiral3},\ref{fig:spiral6}} .
   \begin{figure}[H]
    \centering
      \centering
      \subfloat[Second step]
      {\includegraphics[width=0.3\textwidth]{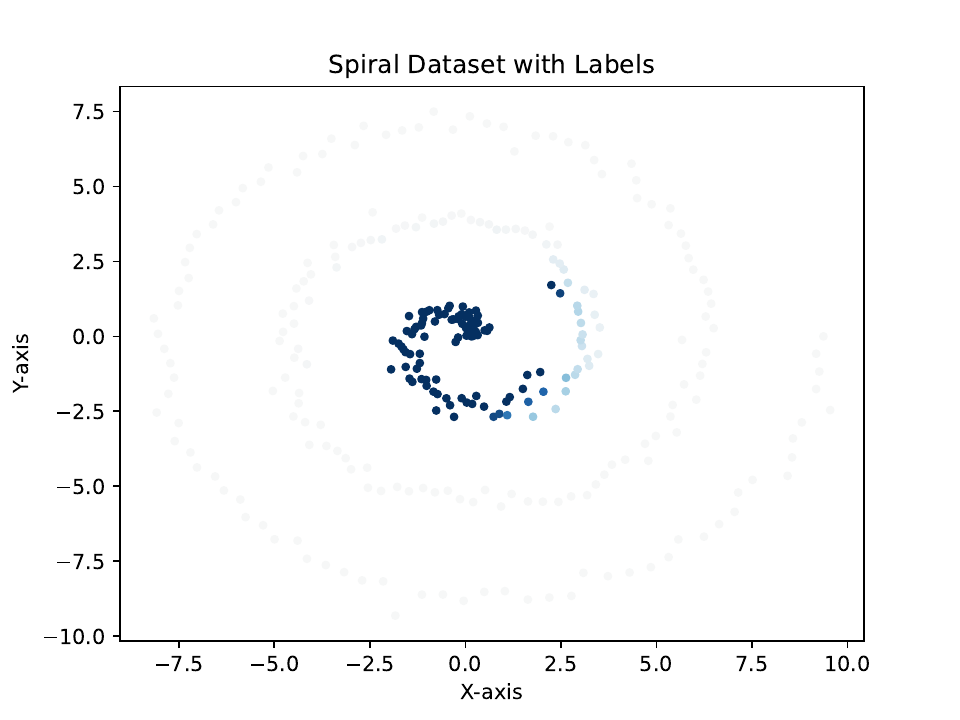}\label{fig:spiral2}} 
      \subfloat[Third step]
      {\includegraphics[width=0.3\textwidth]{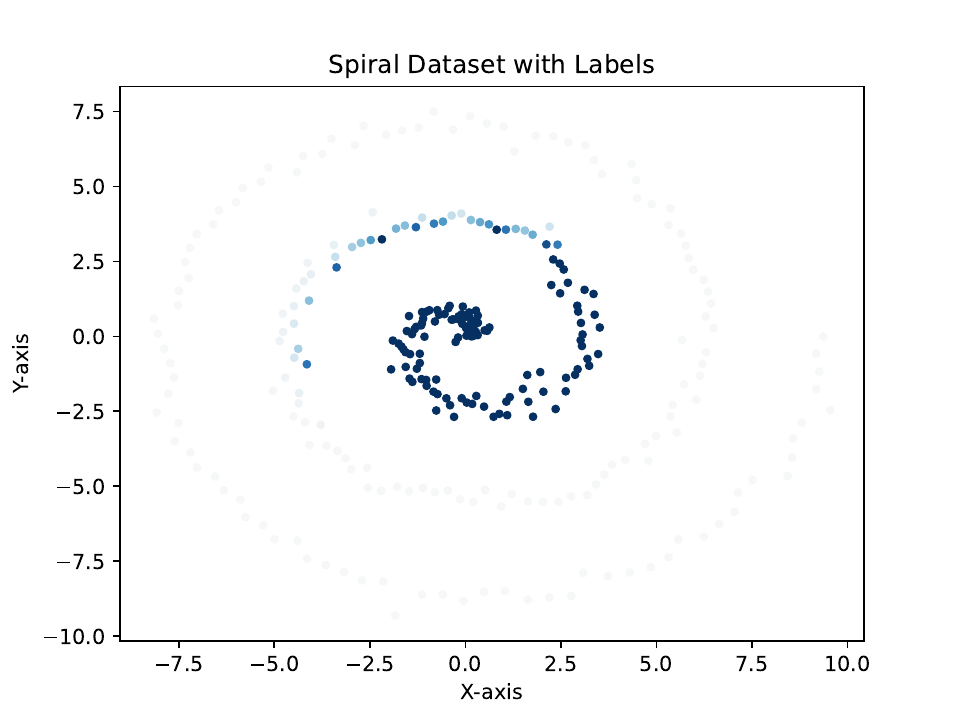}\label{fig:spiral3}}
      \subfloat[6th step]
      {\includegraphics[width=0.3\textwidth]{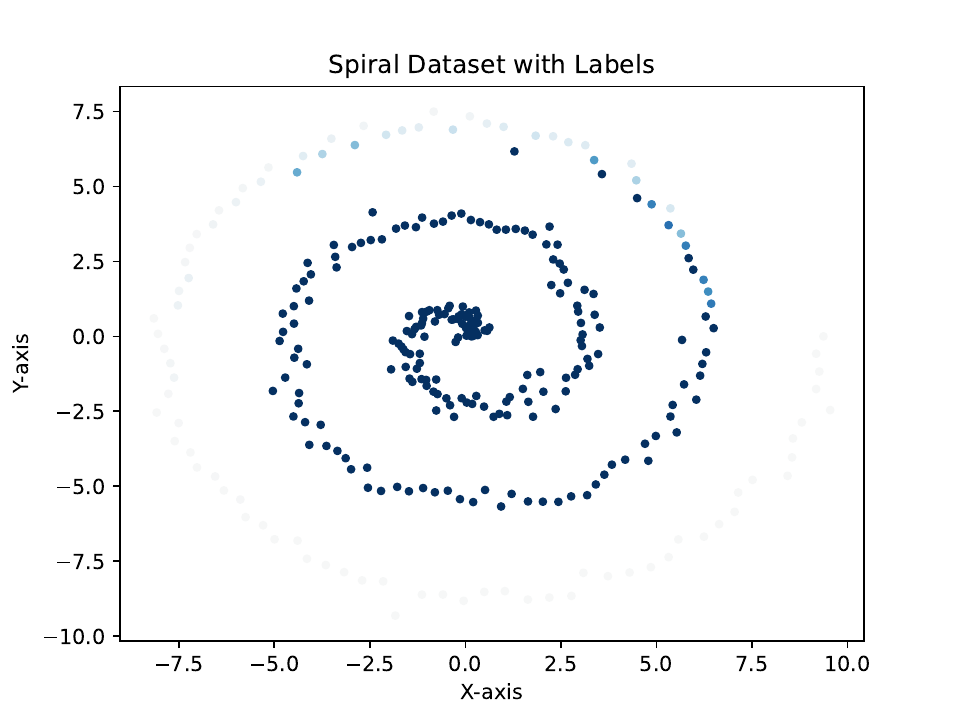}\label{fig:spiral6}}
    \caption{Diffusion process on a helix-shaped dataset formed by 300 points}
  \end{figure}
  The central data point is labeled as 1, while the rest are labeled as 0. It can be observed that when the transition matrix is generated from an adjacency matrix derived from Euclidean distances, the labels spread outward in a circular pattern. The darker colors indicate labels closer to 1.
  Upon observing in the counterclockwise direction, some data points have already been colored in dark blue, while preceding points remain light blue. This suggests that the diffusion process does not occur on the sub-manifold where the dataset lies but rather takes place in the ambient space of the sub-manifold. Despite its high computational efficiency, this algorithm may not perform well in certain classification tasks because it does not consider the intrinsic geometric structure within the sub-manifold. 

  As is illustrated in {\bf Figure \ref{fig:subfig7},\ref{fig:subfig8},\ref{fig:subfig9}}.
  \begin{figure}[htbp]
    \centering
      \centering
      \subfloat[First Step]
      {\includegraphics[width=0.3\textwidth]{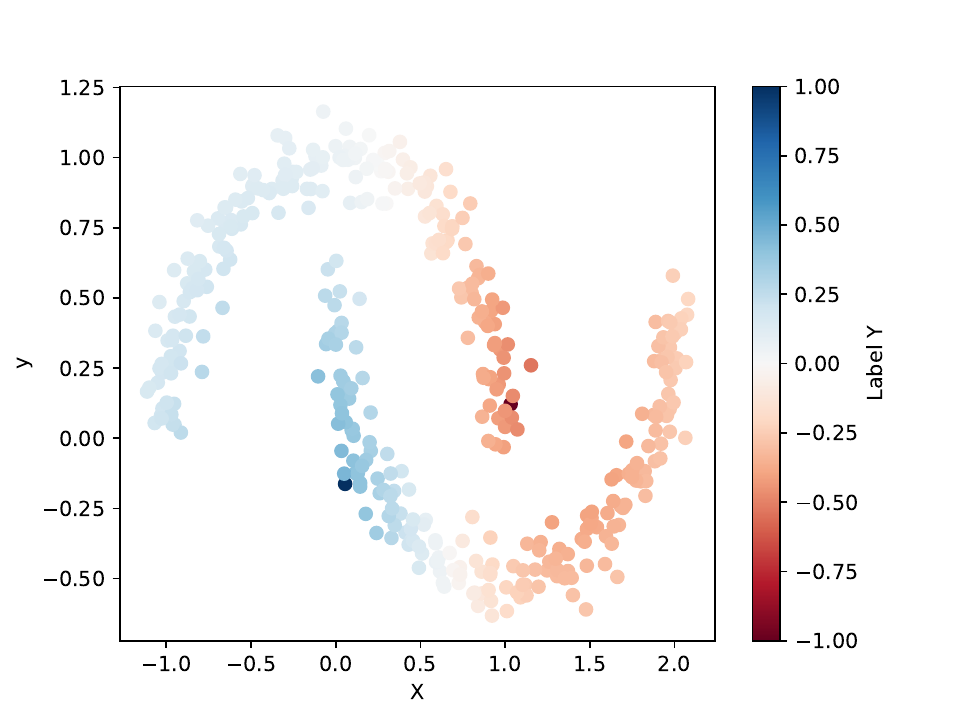}\label{fig:subfig7}}    
      \subfloat[Second Step]
      {\includegraphics[width=0.3\textwidth]{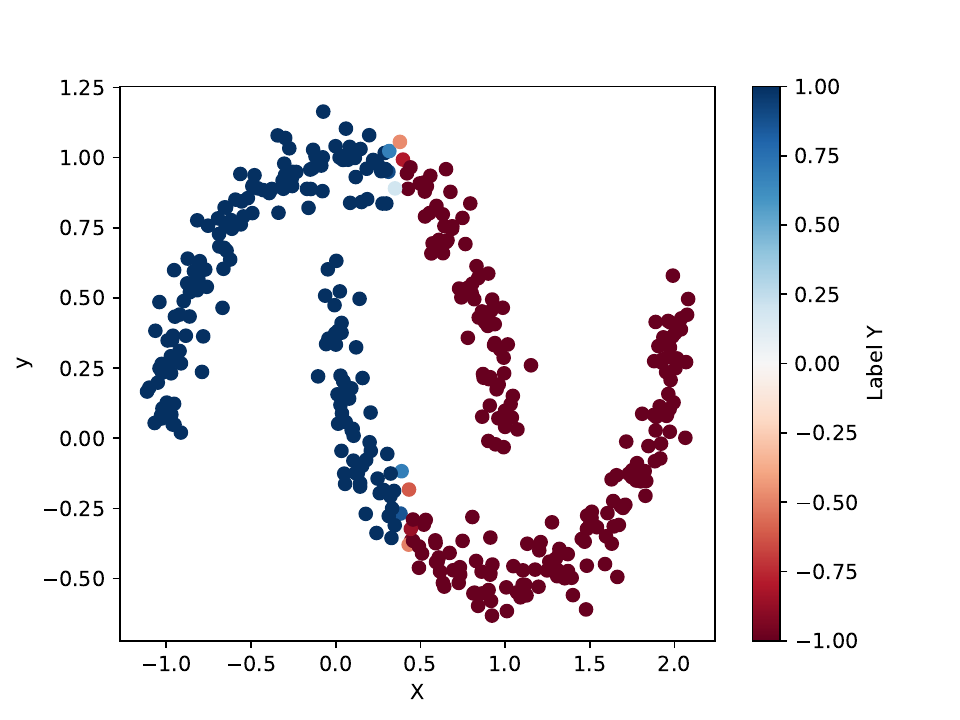}\label{fig:subfig8}}
      \subfloat[Third Step]
      {\includegraphics[width=0.3\textwidth]{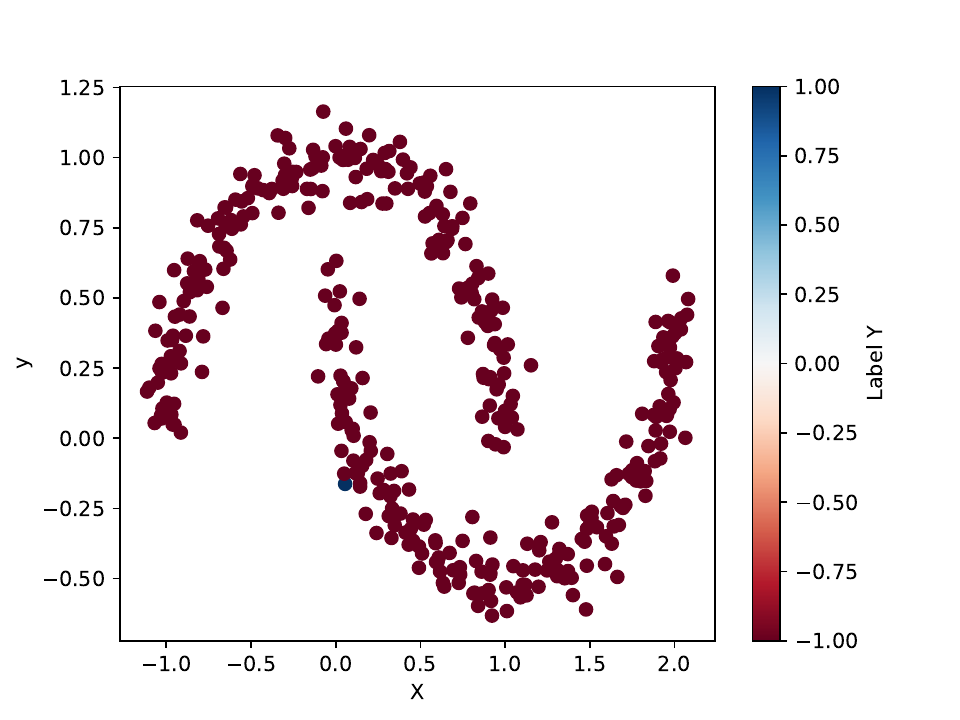}\label{fig:subfig9}}
    \caption{The label propogation process based on transition matrix constructed by Euclidean distance}
  \end{figure}
  Here we generated a two-moons shaped set $X$ consisting of $400$ points of data in the plane. The red-colored data points were labeled by $-1$, and the blue-colored data points were labeled by $+1$. Initially, we randomly chose one data point on each side of this dataset as the  elements of labeled data $X_{l}$, and other elements of $X$ belonged to the unlabeled data. These figures illustrate the first three steps  of diffusion. The labels diffuse outward from the labeled data points. The red labels are concentrated around the center of the dataset, influencing both sides simultaneously. On the other hand, the blue data points, located at the dataset's edge, can only diffuse upward. Ultimately, due to the broader influence of the red data points, all data points are labeled as red, which does not achieve the desired classification outcome. Therefore, it is necessary to improve the construction method of the transition matrix, which determines the diffusion of data points. 

   \subsubsection{Improved transition matrix based on geodesic distance}

  A concise idea is to replace Euclidean distance with geodesic distance, which considers the difference between distances on the manifold and straight-line distances between two points. However, constructing an adjacency matrix based on geodesic distance may introduce high computational complexity. Fortunately, it is not necessary to compute the geodesic distance for every pair of points. Depending on the characteristics of different datasets, we only need to establish the geodesic distances between each data point and its k-nearest neighbors. In Appendix A {\bf theorem\ref{theorem}}, we will illustrate that this approach achieves the same effect as the classical model in calculating the geodesic distance between every pair of points to build the transition matrix. 

  The geodesic distance is the length of the shortest curve connecting two points, manifested as the shortest distance between two points on a finite graph(See \cite{li2019geodesic}).Numerically, we can employ the Floyd algorithm to achieve this process. The Floyd–Warshall algorithm, proposed by Floyd and Warshall, is an algorithm with $\mathcal{O}(|V|^3)$ complexity that can rapidly compute the approximate shortest distances of k-nearest neighbors for data points. The specific implementation process is detailed in Algorithm 2(See \cite{Floyd}).

  \begin{algorithm}[htbp]
    \begin{algorithmic}
        \State{Input the adjacency matrix $G$, the number of nearest neibourhoods $k$}
    \State{Initialise the distance matrix: $\mathrm{Dist}(i,j) = +\infty$}
    \For{$i$ in $\mathrm{Vertex}$}
        \State{$\mathrm{Dist}(i,i)=0$}
        \For{$j$ in $\mathrm{Vertex}$}
          \If{$(i,j)$ in $\mathrm{Edge}$}
            \State{$\mathrm{Dist}(i,j)=G(i,j)$}
          \EndIf
        \EndFor
      \EndFor
    \For{$i$ in $\{1,\cdots,k\}$}
      \For{$u$ in $\mathrm{Vertex}$}
        \For{$v$ in $\mathrm{Vertex}$}
          \If{$(u,v)$ in $\mathrm{Edge}$}
            \If{$\mathrm{Dist}(u,i)+\mathrm{Dist}(i,v)<\mathrm{Dist}(u,v)$}
              \State{$\mathrm{Dist}(u,v)=\mathrm{Dist}(u,i)+\mathrm{Dist}(i,v)$}
            \EndIf
          \EndIf
        \EndFor
      \EndFor
    \EndFor
    \end{algorithmic}
    \caption{Floyd-Warshall}
    \label{Alg 2}
    \end{algorithm}
    The algorithm in lines 11 to 21 involves nested triple-loop structures, significantly reducing computational efficiency. We consider replacing the element-wise iteration process with matrix operations. Since matrix operations in Python are parallelized, this allows the algorithm to handle large-scale datasets efficiently.

    \begin{algorithm}[htbp]
     \begin{algorithmic}
          \State{Input the adjacency matrix $G$, the number of nearest neibourhoods $k$}
      \State{Initialise the distance matrix: $\mathrm{Dist}(i,j) = +\infty$}
      \For{$i$ in $\mathrm{Vertex}$}
          \State{$\mathrm{Dist}(i,i)=0$}
          \For{$j$ in $\mathrm{Vertex}$}
            \If{$(i,j)$ in $\mathrm{Edge}$}
              \State{$\mathrm{Dist}(i,j)=G(i,j)$}
            \EndIf
          \EndFor
        \EndFor
      \For{$i$ in $\{1,\cdots,k\}$}
        \State{$\mathrm{Dist} = \min\{\mathrm{Dist},\mathbf{1}\otimes \mathrm{Dist}(:,i)+\mathrm{Dist}(i,:)\otimes \mathbf{1}\}$}
      \EndFor
     \end{algorithmic}
      \caption{Our Improved Floyd-Warshall}
      \label{Alg 3}
      \end{algorithm}

      By inputting the geodesic distances constructed based on the Floyd algorithm into the Diffusion Map algorithm, a new transition matrix can be obtained. We will first study the performance of the improved algorithm on the helix-shaped dataset.
      \begin{figure}[htbp]
        \centering
          \centering
          \subfloat[27th step]
          {\includegraphics[width=0.3\textwidth]{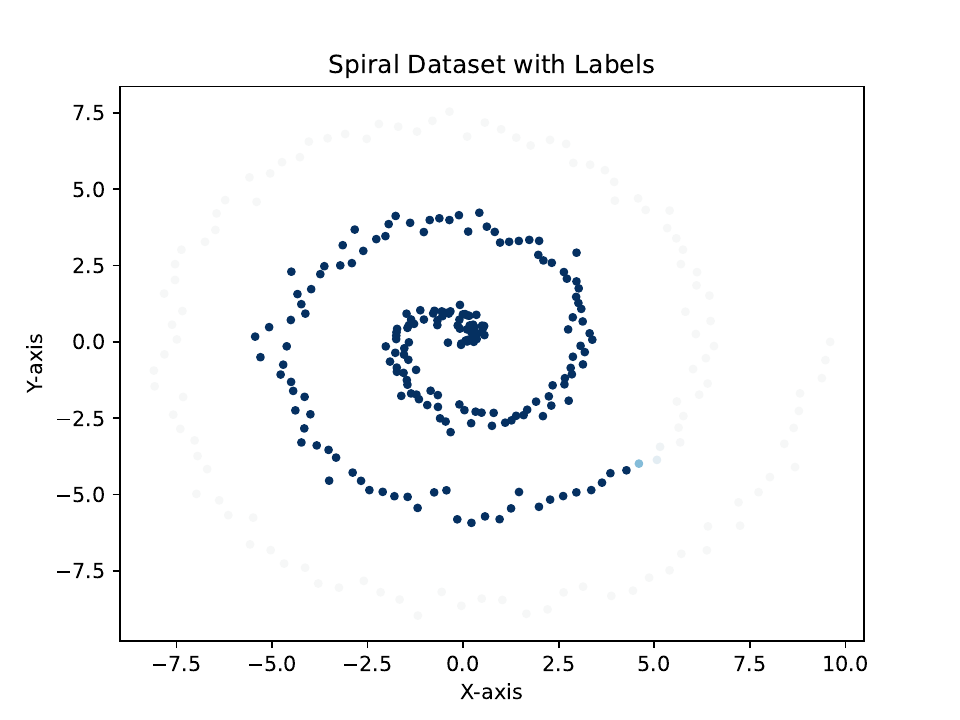}\label{fig:SpiralFloyd27}}    
          \subfloat[46th Step]
          {\includegraphics[width=0.3\textwidth]{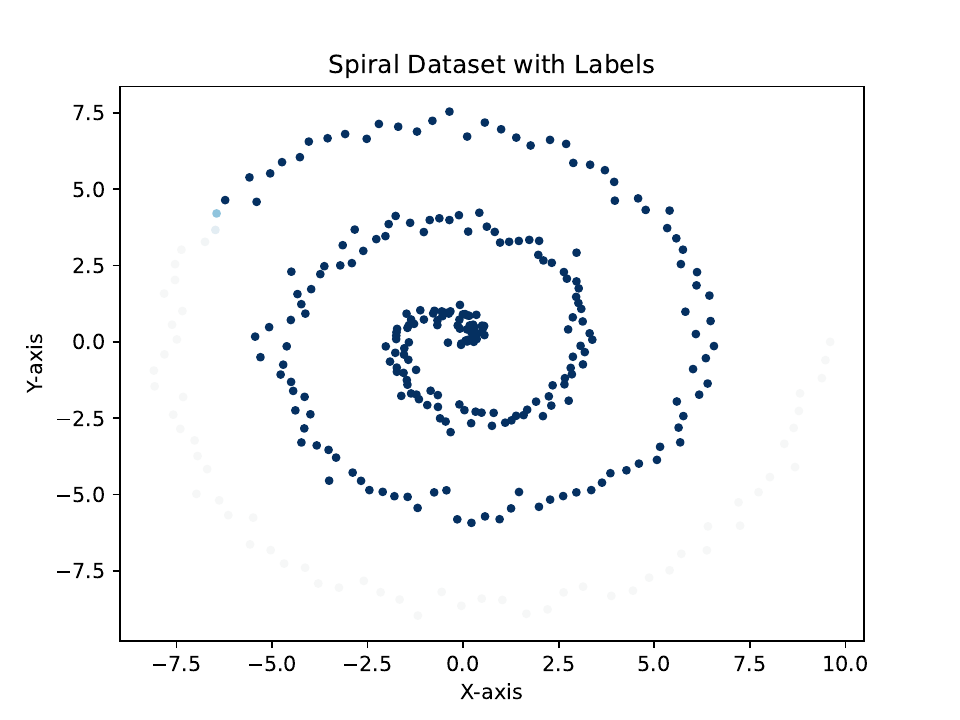}\label{fig:SpiralFloyd46}}
          \subfloat[69th Step]
          {\includegraphics[width=0.3\textwidth]{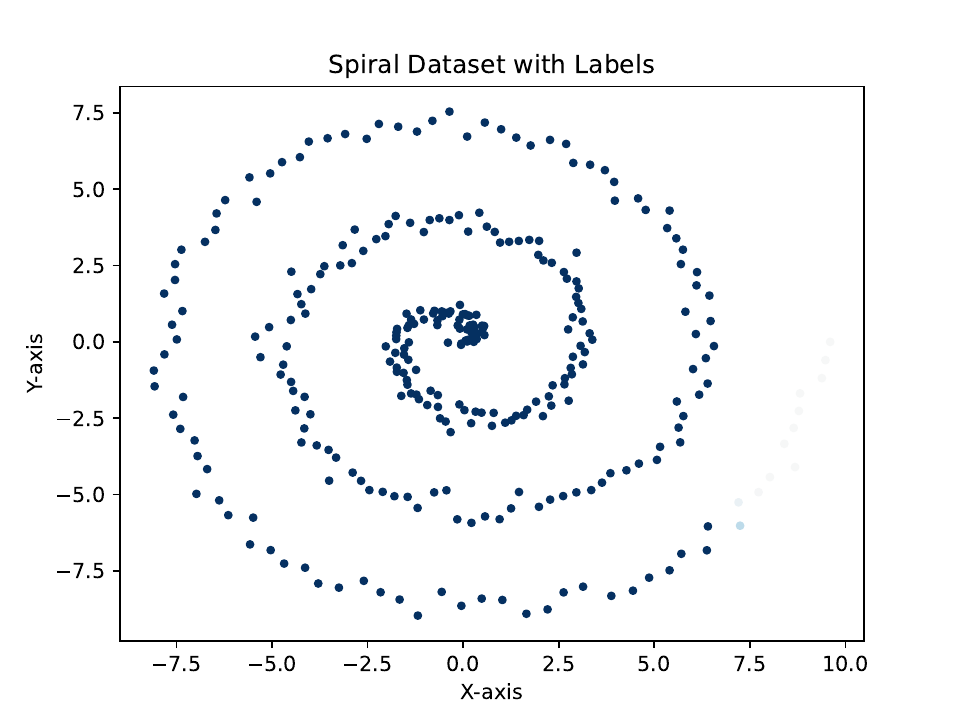}\label{fig:SpiralFloyd69}}
        \caption{The label propagation process based on transition matrix constructed by geodesic distance}
      \end{figure}
      {\bf Figures \ref{fig:SpiralFloyd27}, \ref{fig:SpiralFloyd46}, \ref{fig:SpiralFloyd69}} illustrates the diffusion process at steps 27, 46, and 69 determined by the improved transition matrix under the same conditions. It can be observed that although the diffusion speed has significantly slowed down, taking 69 diffusion steps to label all data points, the color intensity of the labeled data changes continuously. This indicates that the diffusion process occurs counterclockwise in the helix rather than the ambient space, effectively reflecting the manifold's geometric structure. 

      The improved algorithm also exhibits superior performance in classification tasks. As shown in {\bf Figure \ref{fig:TwomoonFloyd1}, \ref{fig:TwomoonFloyd2}, \ref{fig:TwomoonFloyd3}} under the same conditions, with the influence of the improved transition matrix, all points above the dataset are colored red, while those below are all blue. This indicates that, after a certain period, the values of the label propagation function on both sides of the dataset have opposite signs, effectively distinguishing the dataset into two distinct regions.

      \begin{figure}[htbp]
        \centering
          \centering
          \subfloat[First Step]
          {\includegraphics[width=0.3\textwidth]{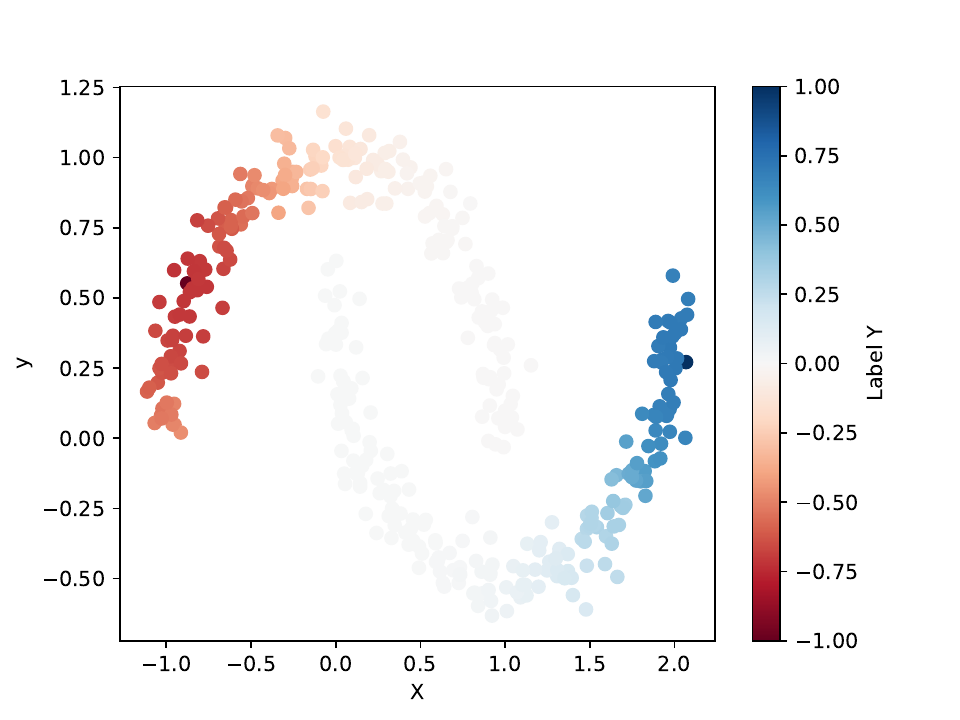}\label{fig:TwomoonFloyd1}}    
          \subfloat[Second Step]
          {\includegraphics[width=0.3\textwidth]{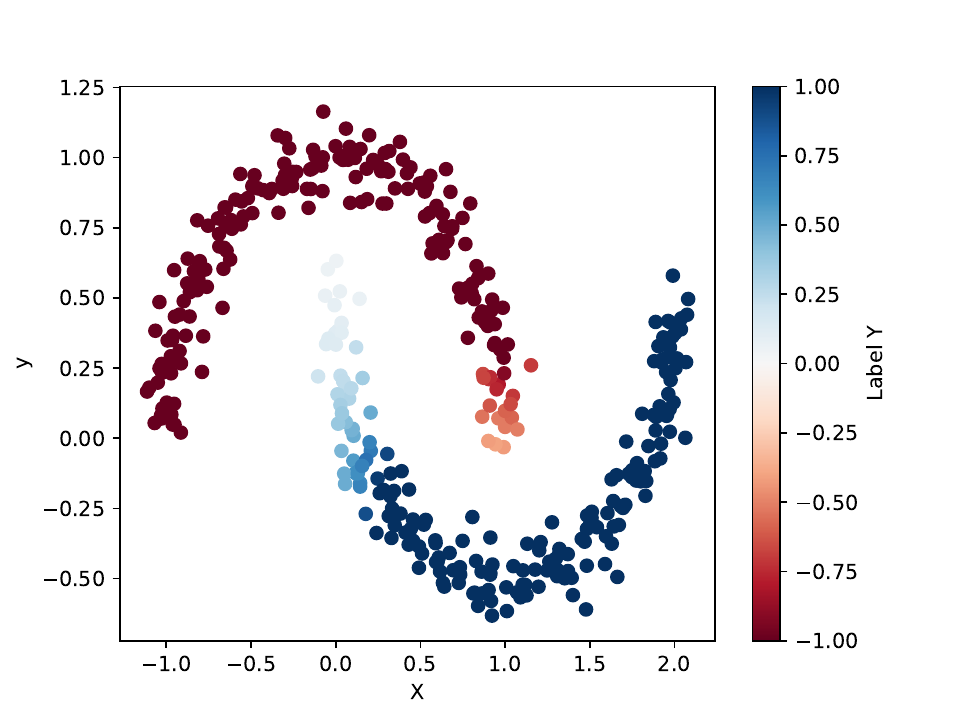}\label{fig:TwomoonFloyd2}}
          \subfloat[Third Step]
          {\includegraphics[width=0.3\textwidth]{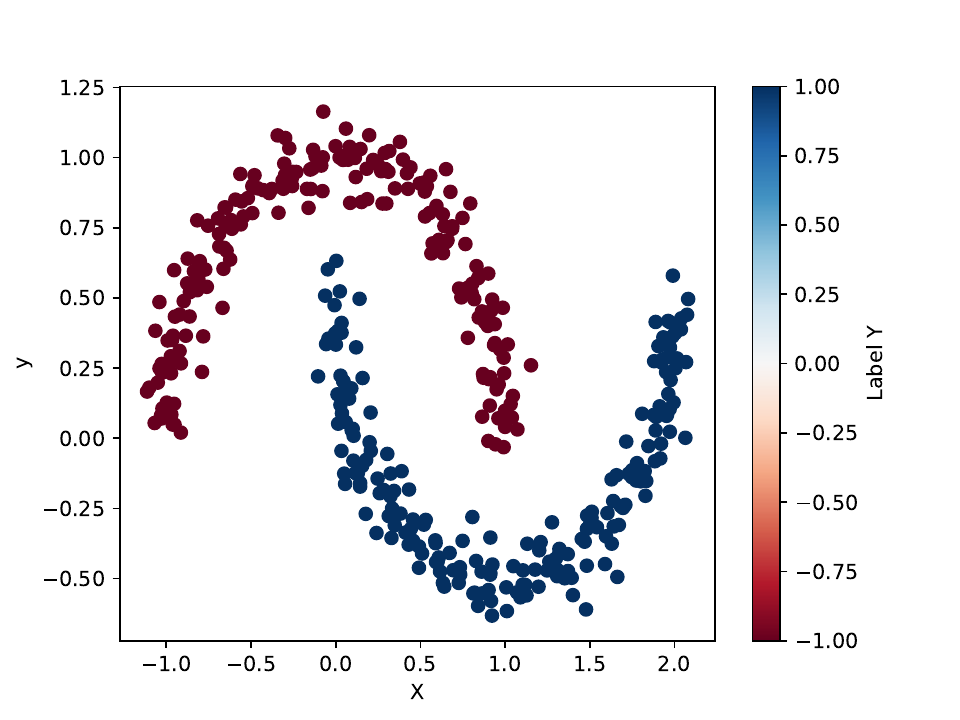}\label{fig:TwomoonFloyd3}}
        \caption{The label propogation process based on transition matrix constructed by geodesic distance}
      \end{figure}

\subsection{Neumann Heat Kernel(NHK) manifold regularization model}
Inspired by the heat diffusion process, we have constructed a label propagation model and simulated the label propagation process on the dataset using the improved diffusion mapping algorithm. Experiments show that this effectively increased the number of labeled data points on the manifold. However, in reality, the diffusion process should not be limited to the dataset alone but should occur in the entire space. After enough time, the label distribution in the entire space will reach a stable state, with some data labeled as positive and others as negative. Due to the smoothness of the label propagation function, we can inevitably select a series of labels as 0, forming a classification plane in space. In this sense, the manifold regularization classification model is essentially a prediction of the label distribution function when it reaches a steady state. Based on this idea, we will construct a penalty that measures the stability of the label propagation process after a sufficiently long time. We sum up the above discussion and point out the following theorem:

\begin{theorem}
  $\mathcal{M}$ is a $d$ dimensional compact sub-manifold in $\mathbb{R}^{n}$, when $\mathrm{Ric}(\mathcal{M})\geqslant 0$, that is, for any $w\in T\mathcal{M}$,
  $\mathrm{Ric}(w,w)\geqslant 0$. Then there exists function $f(x)$ such that 
  $$\lim_{t \to \infty}u(x,t) = f(x)$$
\end{theorem}
This theorem indicates that the distribution function of the label exists after a sufficiently large time, and the proof of it can be seen in Appendix B.

Let $f(x)$ be a label propagation function after a sufficiently large time. As this function already represents the label distribution when label propagation reaches a steady state, we apply the operator $e^{t\Delta}$ to the function, examining the diffusion process for another period. We expect the distance between $e^{t\Delta}f$ and $f$ to be sufficiently small. Naturally, we consider adding a penalty term to the distance between $e^{t\Delta}f$ and $f$ Thus, we construct the following penalty:
$$\mathcal{P}(f) =  \gamma_{I}\left\lVert f-e^{t\Delta}f \right\rVert _{\mathcal{H}(\mathcal{M})}^2=\gamma_{I}\int_{\mathcal{M}}(f-e^{t\Delta}f)^2\,d\mu  $$
 Based on the above discussion, we present the  manifold regularization model based on Neumann Heat Kernel(NHK):
 \begin{gather}
  f^{*}=\arg\min_{f\in\mathcal{H}_{K}}\frac{1}{l}\sum_{i=1}^{l}V(x_{i},y_{i},f) + \gamma_{A}\left\lVert f \right\rVert _{K}^2 + \gamma_{I}\int_{\mathcal{M}}(f-e^{t\Delta}f)^2\,d\mu 
 \end{gather}
 To show that this penalty can lead to a optimizer which expression satisfies the representation theorem, it suffices to show $I-e^{t\Delta}$ is bounded, because 
 $$\mathcal{P}(f)=\gamma_{I}\left\langle (I-e^{t\Delta})f,(I-e^{t\Delta})f \right\rangle _{L^2}=\left\langle f,(I-e^{t\Delta})^{*}(I-e^{t\Delta})f \right\rangle _{L^2}  $$
. Since $\left\lVert A^{*}A \right\rVert =\left\lVert A \right\rVert ^2 $(See \cite{yosida2012functional}), we obtained that
$$\left\lVert (I-e^{t\Delta})^{*}(I-e^{t\Delta}) \right\rVert =\left\lVert I-e^{t\Delta} \right\rVert   ^2$$

 \begin{corollary}
  $I-e^{t\Delta}$ is a bounded operator
 \end{corollary}
In fact, it is a direct corollary of that $\left\lVert e^{t\Delta} \right\rVert \leqslant 1 $  Under the $L^2(\mathcal{M})$ norm(See \cite{contractionsemigroup}).

 We have defined the label propagation function in 2.1 and studied the diffusion process on the dataset using it. Through the above discussion, we do not need to determine the diffusion behavior of the extended propagation function on the data manifold. Instead, we only need to use the manifold regularization model to obtain the distribution of the propagation function on the manifold when it reaches a steady state, thereby achieving the classification effect. We use the improved probability transition matrix to estimate the heat kernel operator. Due to the fitness of the dataset, we employ Monte Carlo methods to estimate integrals. Therefore, the discrete form of the manifold regularization model can be expressed as follows:
 \begin{gather}
  \mathbf{f}^{*}=\arg\min_{\mathbf{f}\in\mathbb{R}^{l+u}}\frac{1}{l}\sum_{i=1}^{l}V(x_{i},y_{i},\mathbf{f}) + \gamma_{A}\left\lVert f \right\rVert _{K}^2  + \frac{\gamma_{I}}{(l+u)^2}\mathbf{f}^{T}(I-P_{G,\varepsilon}^{\frac{t}{\varepsilon}})^{T}(I-P_{G,\varepsilon}^{\frac{t}{\varepsilon}})\mathbf{f}
 \end{gather}
 
\subsection{Neumann Heat Kernel Regularized Least Squares(NHKRLS)}
Now, we illustrate the performance of the improved model using the least squares method within the manifold regularization framework as an example. The loss function is 
$$\mathbf{f}^{*}=\arg\min_{\mathbf{f}\in\mathbb{R}^{l+u}}\frac{1}{l}\sum_{i=1}^{l}(f(x_{i})-y_{i})^2 + \gamma_{A}\left\lVert f \right\rVert _{K}^2  + \frac{\gamma_{I}}{(l+u)^2}\mathbf{f}^{T}(I-P_{G,\varepsilon}^{\frac{t}{\varepsilon}})^{T}(I-P_{G,\varepsilon}^{\frac{t}{\varepsilon}})\mathbf{f}$$
Here we choose the regularized norm as $2$-norm, and we can use representation theorem to show the analytical solution of this optimization problem. Suppose that 
$$f^{*} = \sum_{i=1}^{l+u}{\alpha_{i}K(x_{i},\cdot)}$$
and $\alpha = (\alpha_{1},\cdots,\alpha_{l+u})$, $Y=(y_{1},\cdots,y_{l},0,\cdots 0)$.\\ $J$ is a $(l+u)\times(l+u)$diagonal matrix: $J=\begin{pmatrix}
  I_{l}& \\
  & O_{u}
\end{pmatrix}$, the loss function is 
$$loss(\alpha) = \frac{1}{l}(Y-JK\alpha)^{T}(Y-JK\alpha) + \alpha^{T}K\alpha + \alpha^{T}K^{T}(I-P_{G,\varepsilon}^{\frac{t}{\varepsilon}})^{T}(I-P_{G,\varepsilon}^{\frac{t}{\varepsilon}})K\alpha$$
which leads to the following solution:
$$\alpha^{*} = (JK + \gamma_{A}lI +\frac{\gamma_{I}l}{(l+u)^2}(I-P_{G,\varepsilon}^{\frac{t}{\varepsilon}})^{T}(I-P_{G,\varepsilon}^{\frac{t}{\varepsilon}})K)^{-1}Y$$
We can organize the above discussion into pseudo-code as follow {\bf Algorithm \ref{Alg:NHKRLS}}:
\begin{algorithm}[htbp]
  \begin{algorithmic}
      \State{Input $X_{l}$, $X_{u}$, $Y$ diffusion steps $t$, parameters: $\gamma_{A}$, $\gamma_{I}$, $\varepsilon$}
  \State{Compute the kernel function $K$}
  \State{Compute the transition matrix $P$ according to algorithm1}
  \For{k = 0 $\cdots$ t}
       \State{$u = P_{\varepsilon}^{\frac{1}{\varepsilon}} u$}
       \State{$u|_{X_{l}}(x,t) = u|_{X_{l}}(x,0)$}
      \EndFor
  \State{Recompute the labeled and unlabeled datasets: $X_{l}$, $X_{u}$, labels $Y$}
  \State{Compute $\alpha^{*}$ by $$(JK + \gamma_{A}lI +\frac{\gamma_{I}l}{(l+u)^2}(I-P_{G,\varepsilon}^{\frac{t}{\varepsilon}})^{T}(I-P_{G,\varepsilon}^{\frac{t}{\varepsilon}})K)^{-1}Y$$}
  \State{Output classifier $f^{*}(x) = \sum_{i=1}^{l+u}\alpha_{i}^{*}K(x_{i},x)$}
  \end{algorithmic}
  \caption{Neumann Heat Kernel Regularized Least Squares}
  \label{Alg:NHKRLS}
  \end{algorithm}

  Lines 1 to 7 update the values of the label propagation function on the dataset after t steps of diffusion. Lines 7 to 10, based on the heat balance condition in space, update the values of the extended label propagation function on the data manifold using the values of the label propagation function on the dataset. Here, $f^{*}$ is essentially $\lim_{t \to \infty}u(t,\cdot)$.

  \section{Numerical Experiments}%
    \subsection{Performance of NHKRLS in generated dataset}
  To demonstrate the performance of the model, we compare the model (NHKRLS) with manifold regularization based on the Laplace operator (LapRLS) on three datasets. As shown in  {\bf Figure:\ref{fig:NHKRLS1}...\ref{fig:LapRLS3}}, we generated double-moon, circular, and two-cluster datasets, each containing 400 sample points. These points are divided into two classes based on their positions, and only one sample point in each class is labeled, shown as diamonds in the figure. Red and blue colors represent different categories. We trained the NHKRLS and LapRLS models separately using the above sample points, with the NHKRLS model considering only one step of diffusion, and the remaining parameters appropriately selected. To evaluate the model performance, we randomly generated 6000 sample points in space and used the trained models to classify these sample points. As shown in the figure, blue and gray represent different output labels in the test set, and the blue-gray boundary is the classification hyperplane on the data manifold.
  \begin{figure}[h]
    \centering
      \centering
      \subfloat[NHKRLS twoMoons]
      {\includegraphics[width=0.3\textwidth]{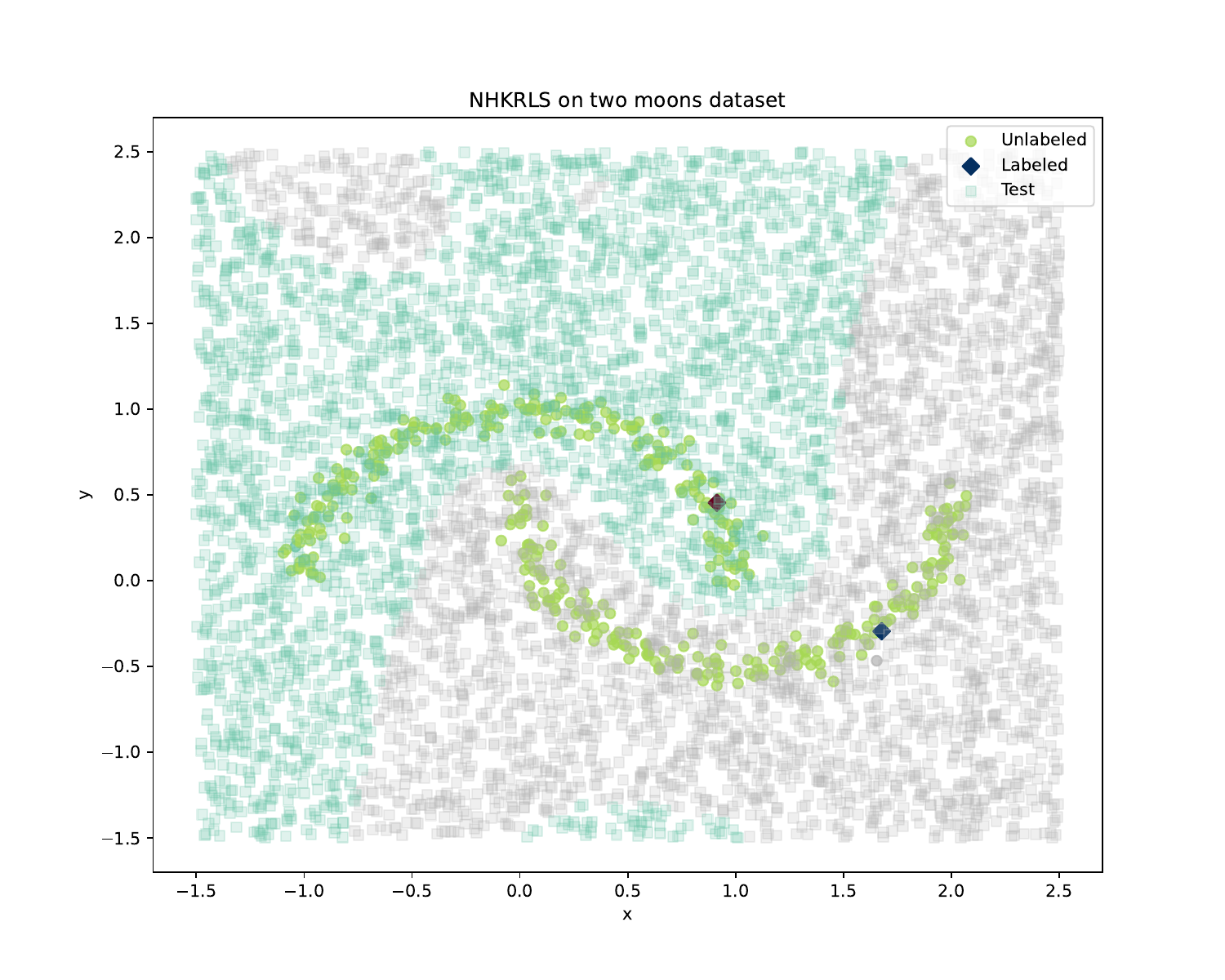}\label{fig:NHKRLS1}}    
      \subfloat[NHKRLS Ring]
      {\includegraphics[width=0.3\textwidth]{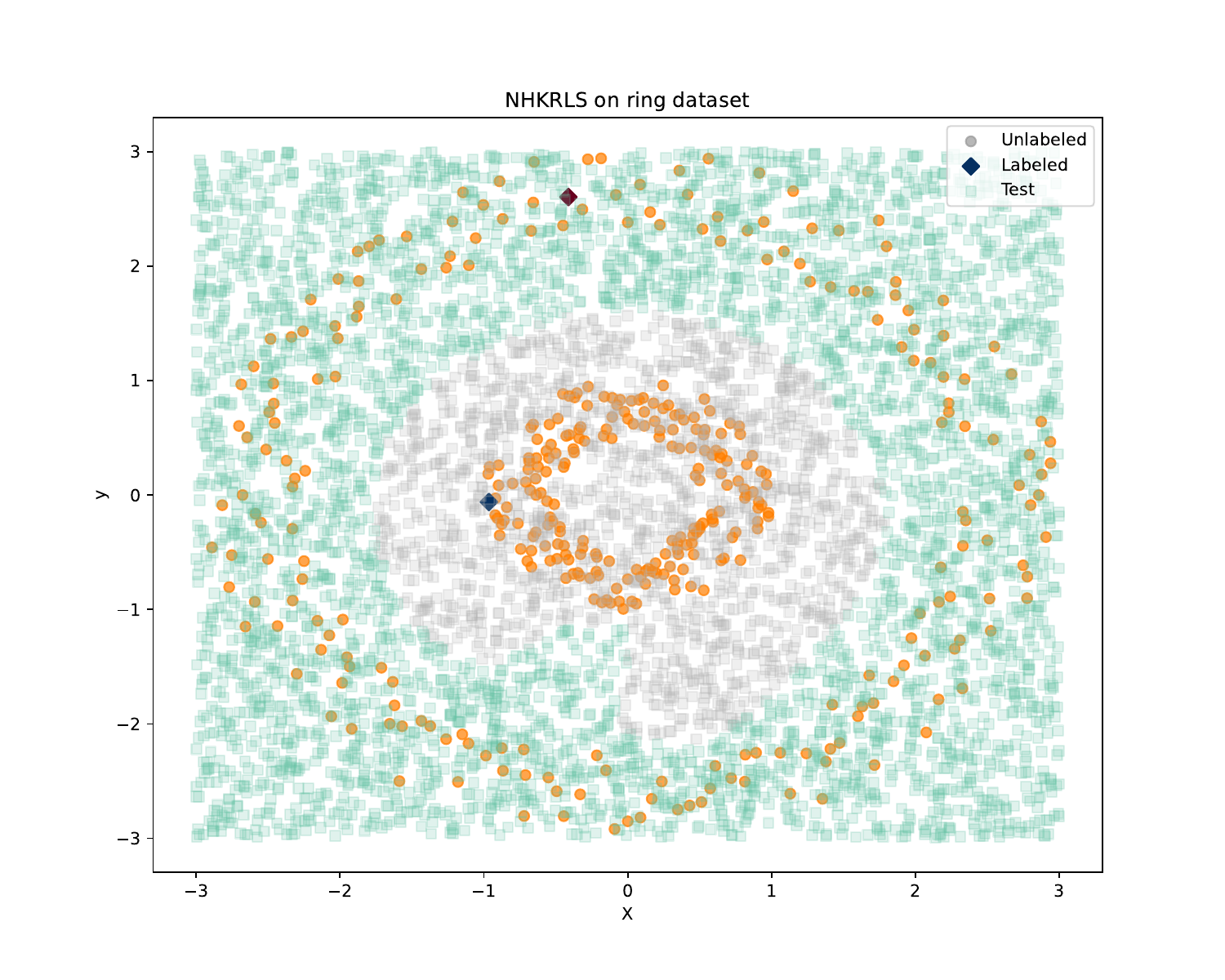}\label{fig:NHKRLS2}}
      \subfloat[NHKRLS twoClusters]
      {\includegraphics[width=0.3\textwidth]{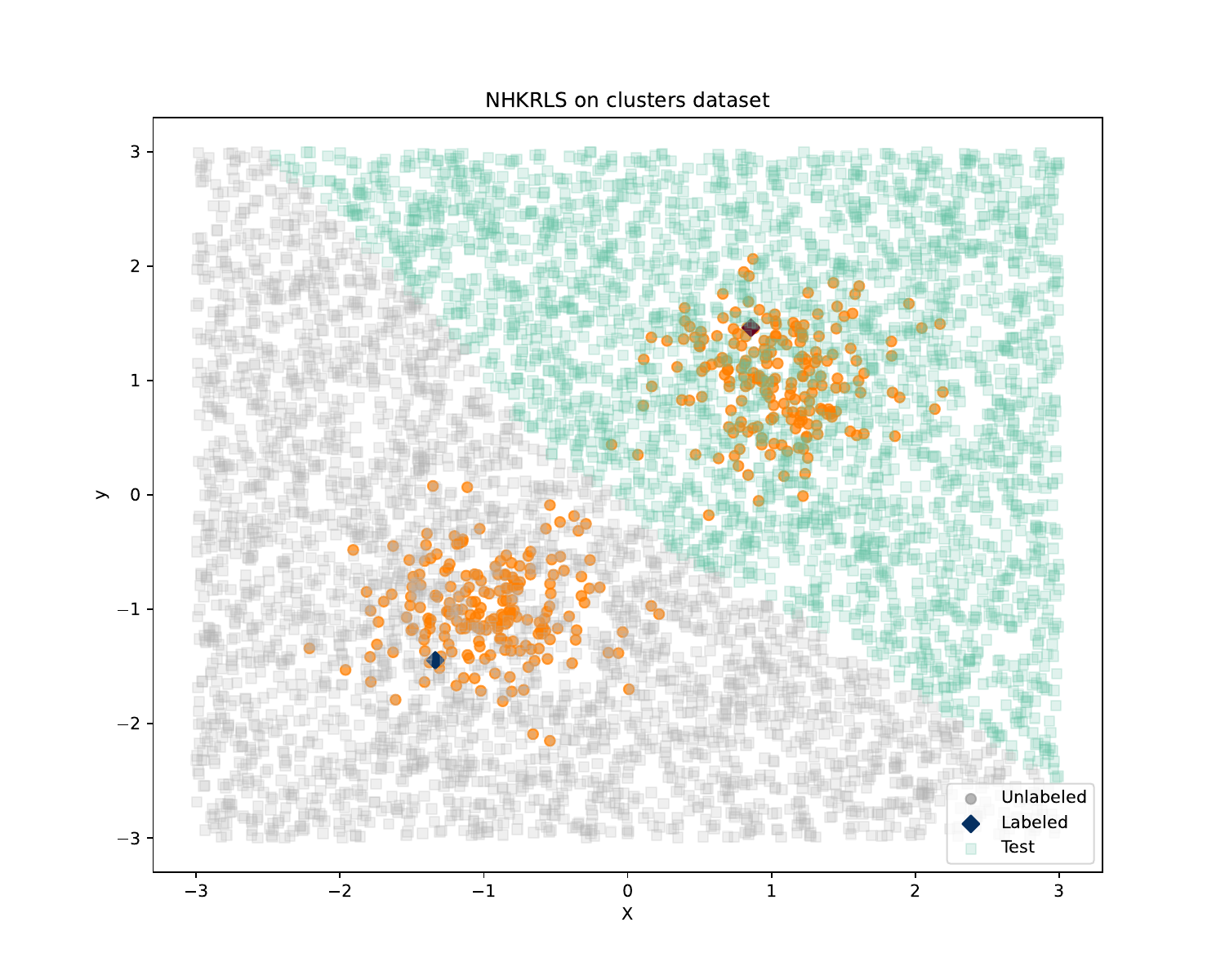}\label{fig:NHKRLS3}}\\
      \subfloat[LapRLS twoMoons]{\includegraphics[width=0.3\textwidth]{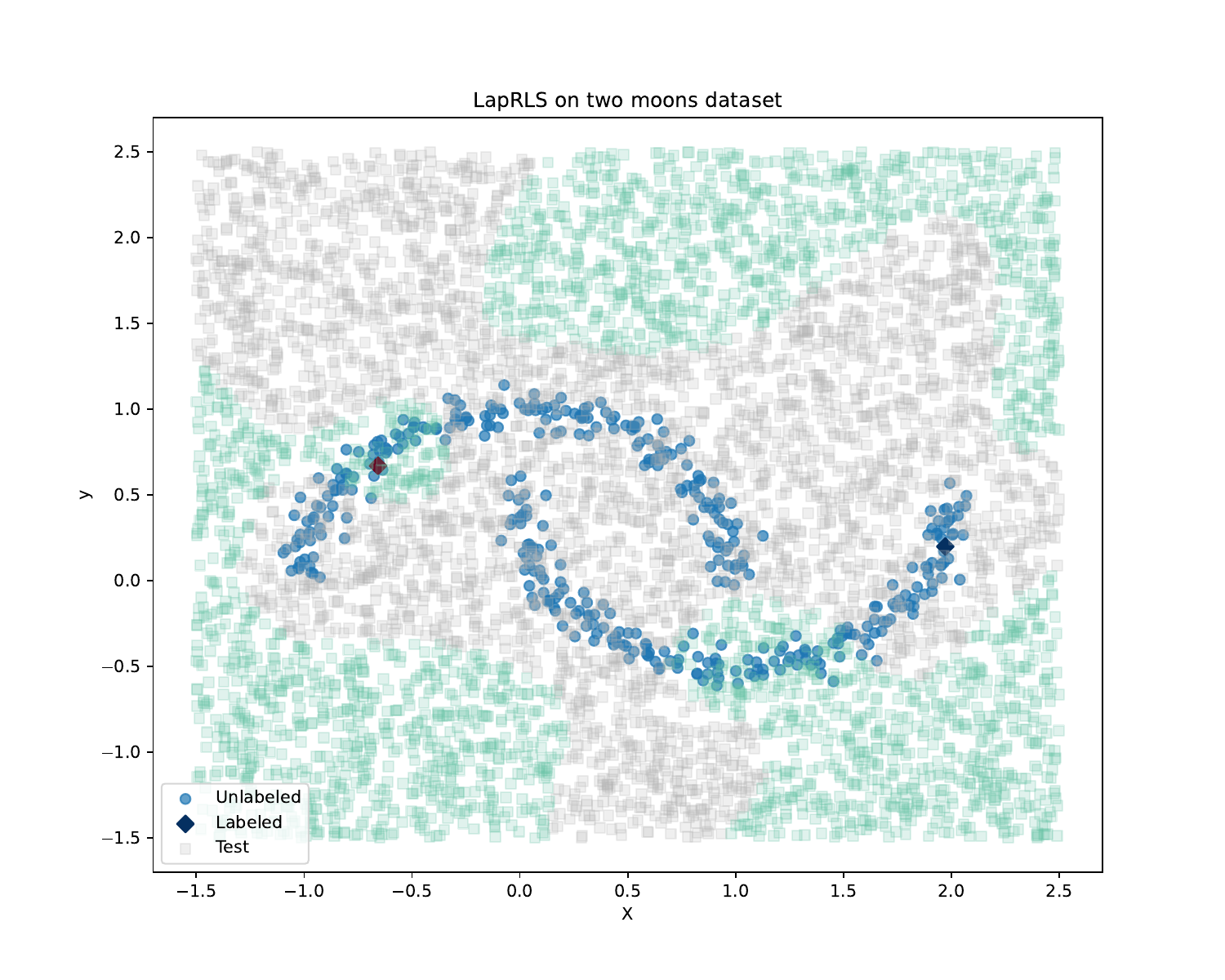}\label{fig:LapRLS1}}
      \subfloat[LapRLS Ring]{\includegraphics[width=0.3\textwidth]{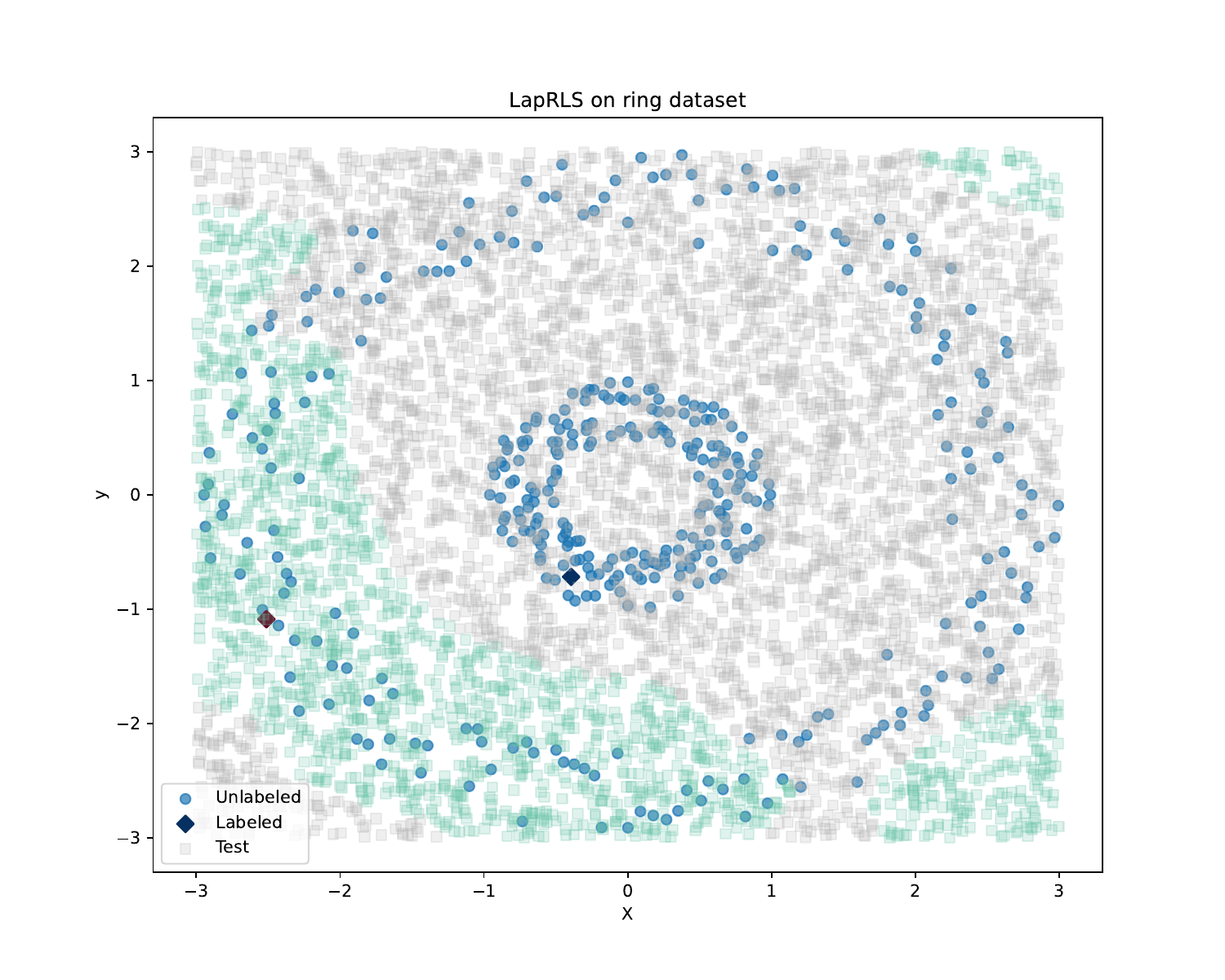}\label{fig:LapRLS2}}
      \subfloat[LapRLS twoClusters]{\includegraphics[width=0.3\textwidth]{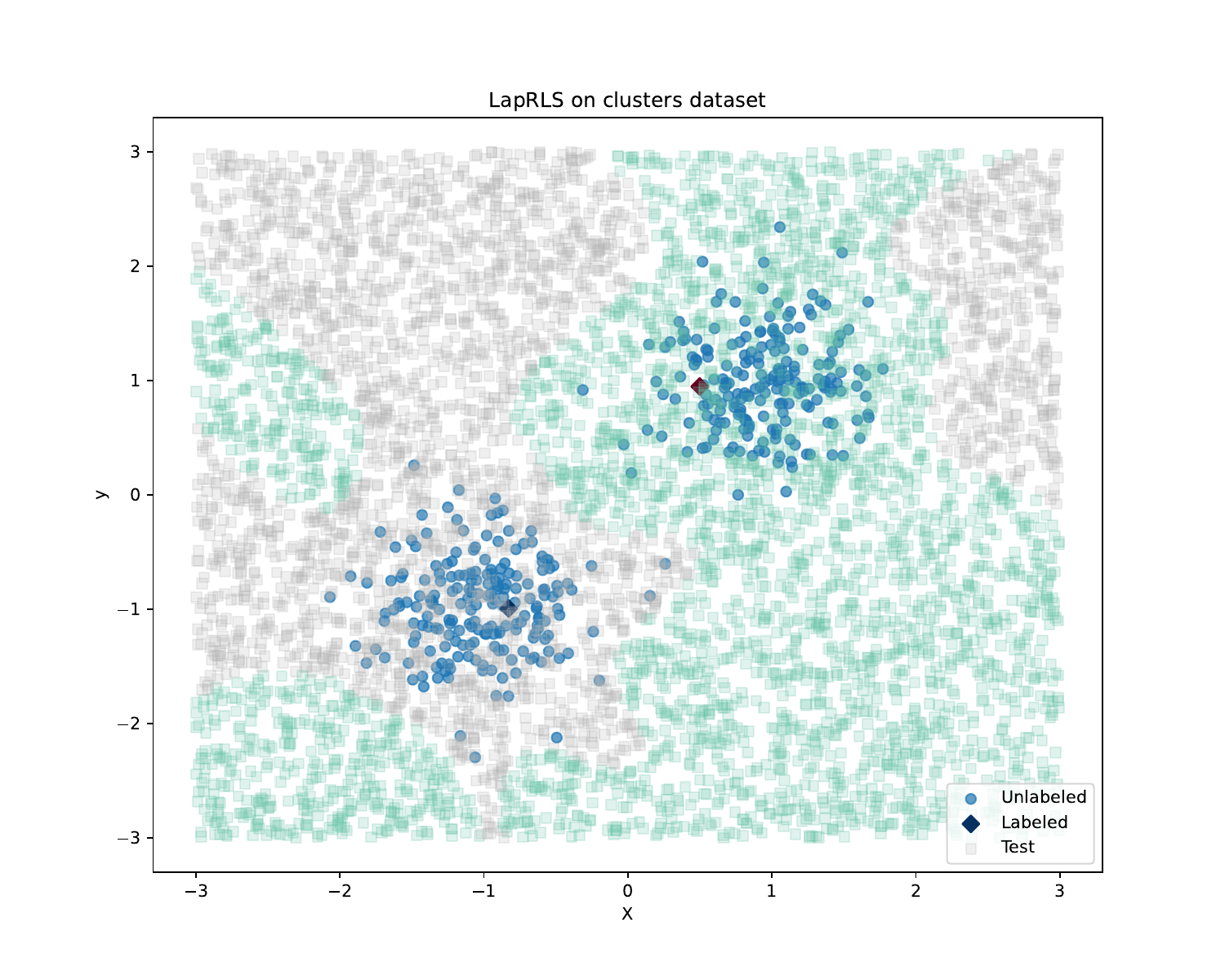}\label{fig:LapRLS3}}
    \caption{The performance of NHKRLS and LapRLS in twoClusters, twomoon and Ring dataset. Labeled dataset is diamond shaped and unlabeled is circular.}
  \end{figure}

  It can be observed that classifiers trained by NHKRLS have relatively smooth classification hyper-planes, and they are at a certain distance from the dataset, distinguishing different classes well. In these three datasets, the first and third datasets are relatively dense, and the second dataset is relatively sparse. The NHKRLS model performs well on all three datasets. In comparison, the classification hyperplane trained by LapRLS exhibits discontinuity (as shown in {\bf Figure \ref{fig:LapRLS1}}) and weak regularity (as shown in {\bf Figure \ref{fig:LapRLS3}}), and the classification performance is inferior to the NHKRLS model. The reason for this phenomenon may be the insufficient number of labeled samples. Although the LapRLS model can correctly classify locally around labeled sample points, its overall performance on the dataset is poor. Training the NHKRLS model involves a diffusion process on the dataset, increasing the number of sample points, and resulting in a better classifier.

  \subsection{NHKRLS on MINIST dataset}
  \subsubsection{Experiment1: Binary classification task}
  In this experiment, we utilized LapRLS, NHKRLS models, and the classical LS model to classify a set consisting of pixel images of handwritten Arabic digits $0$ and $8$. We selected 1000 images of handwritten digits $0$ and $8$ each as the training set, with only one data point labeled in each class, and the remaining data points unlabeled. Both models were configured with parameters $\gamma_{A}=0.00025$, $\gamma_{I}=0.925$, and Gaussian kernel function. The conclusions in this section are similar for parameters within the same order of magnitude. For the NHKRLS algorithm, we chose $10$ diffusion steps, considering only the $5$-nearest neighbors in each diffusion. Since each data point is a $28\times 28$ pixel image, to construct the distance matrix between different data points, we used the Frobenius norm [reference]. We selected 500 unlabeled images of the same handwritten digits with a training set from the test set and performed classification using the classifiers obtained from different models.

  {\bf Fig\ref{fig:Error1}} illustrates the error curves of the three models under different numbers of labeled samples. It can be observed that, with varying numbers of labeled samples, the NHKRLS model consistently demonstrates higher classification accuracy. The LapRLS model shows lower classification accuracy when the number of labeled samples is too low, but its accuracy significantly improves when the number of labeled samples exceeds 32. The LS model performs poorly on this dataset, possibly due to its linearity, whereas the dataset exhibits a non-linear classification plane.
  \begin{figure}[htbp]
    \centering
    \subfloat[Experiment1]{\includegraphics[width=0.5\textwidth]{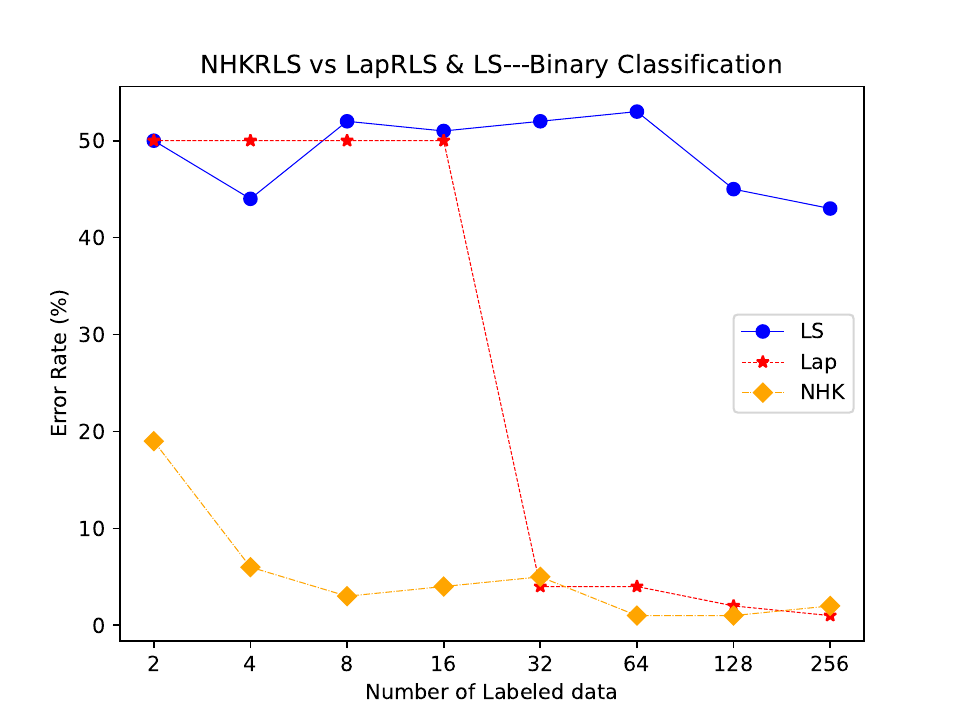}\label{fig:Error1}}
    \subfloat[Experiment2]{\includegraphics[width=0.5\textwidth]{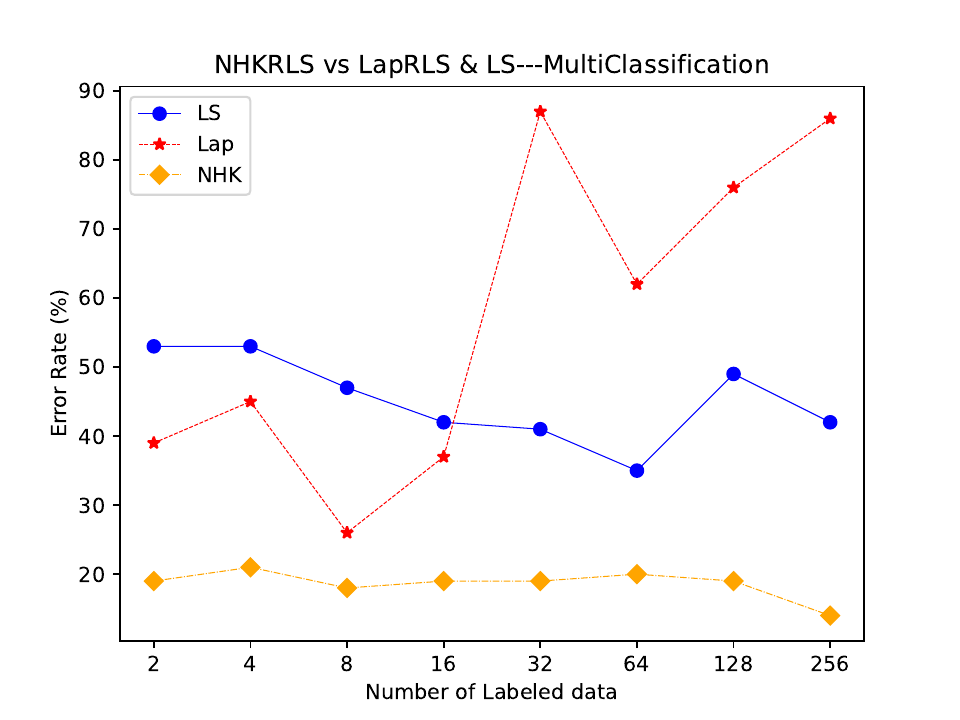}\label{fig:Error2}}
    \caption{Experiment1:The performance of NHKRLS, LapRLS and LS on  binary classification dataset with different numbers of labeled data.\quad Experiment2: The performance of these three model on multi classification dataset with different numbers of labeled data}
  \end{figure}

\subsubsection{Experiment2 Multi-Classification task}
In this experiment, we evaluated the performance of NHKRLS in a multi-classification task. We experimented on the handwritten digit dataset, employing the One-vs-Rest strategy for classification. Initially, we divided the training set into $10$ subsets, each containing $1900$ samples. Half of the samples in each subset represented images of a specific digit from $0$ to $9$, while the remaining samples consisted of an equal amount of other digits. For instance, in the \texttt{X\textunderscore one} set, there were $950$ instances of digit $1$, and each of the digits $0 $to $9$ (excluding $1$) appeared $105$ times, with digit $9$ occurring $110$ times to maintain an equal balance of positive and negative examples. Similar to Experiment $1$, we uniformly labeled $2, 4, \cdots, 256$ data points in both positive and negative classes, while the labels of the remaining points were set to $0$. We trained $10$ classifiers on these 10 subsets, each capable of recognizing a single digit. Let the classifier that identifies digit $i$ be denoted as $f_i$. For a test data point $x$, we represent its output vector $f(x)$ as $(f_1(x), …, f_{10}(x))$, and its final output label is determined as $i_{0} = \arg\max_{i} f(x)$.

We selected the first $100$ elements from the test set and tested the performance of NHKRLS, LapRLS, and LS classifiers in this multi-classification task. As shown in {\bf Fig\ref{fig:Error2}}, the NHK model exhibited stable classification accuracy, around $80\%$, across varying numbers of labeled datasets. The LapRLS model showed relatively high accuracy, around $65\%$, when the sample size was small, but experienced a sharp decline as the sample  size increased. The original paper mentions limited research on the correlation between the model’s generalization error and the variation in the number of labeled samples. We infer that this might be due to the increase in non-representative sample quantity when the number of labeled samples increases, leading to misidentification of geometric features between different data points and resulting in an increase in classification error. The LS model, as a linear classification model, still performed less effectively than the semi-supervised non-linear classification models in this multi-classification task.

To visually emphasize the differences between the two models, we visualized the results of the multi-classification task. Firstly, we built a test set containing 1000 sample points and used classifiers trained with 32 labeled sample points to classify them. We then selected sample points with labels 0 and 8 and applied the diffusion mapping algorithm to reduce them to three dimensions, resulting in the diffusion plot shown in {\bf Fig\ref{Fig:DR}}.
\begin{figure}[ht]
  \centering
  \subfloat[Test]{\includegraphics[width=0.4\textwidth]{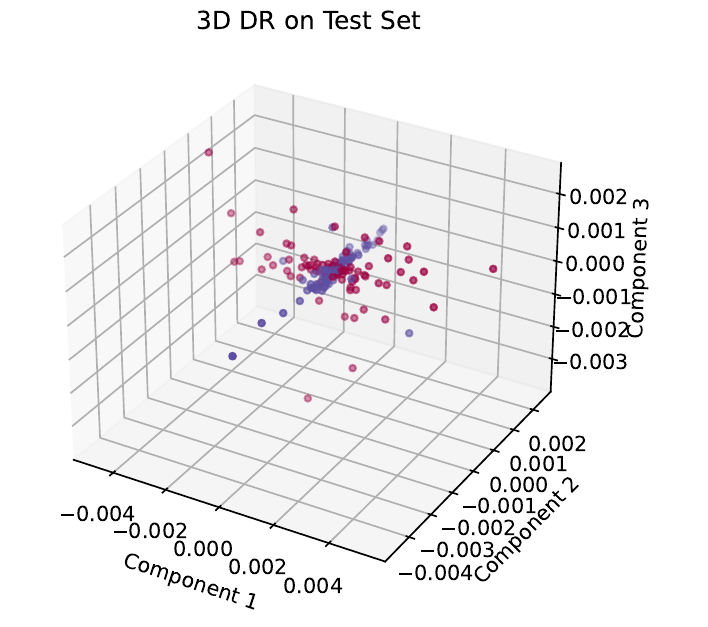}\label{fig:Test}}
  \subfloat[NHK]{\includegraphics[width=0.35\textwidth]{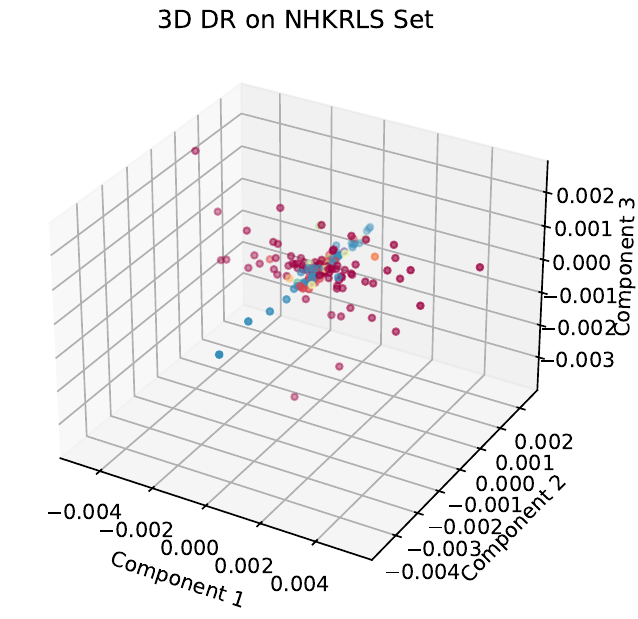}\label{fig:NHK}}\\
  \subfloat[Lap]{\includegraphics[width=0.4\textwidth]{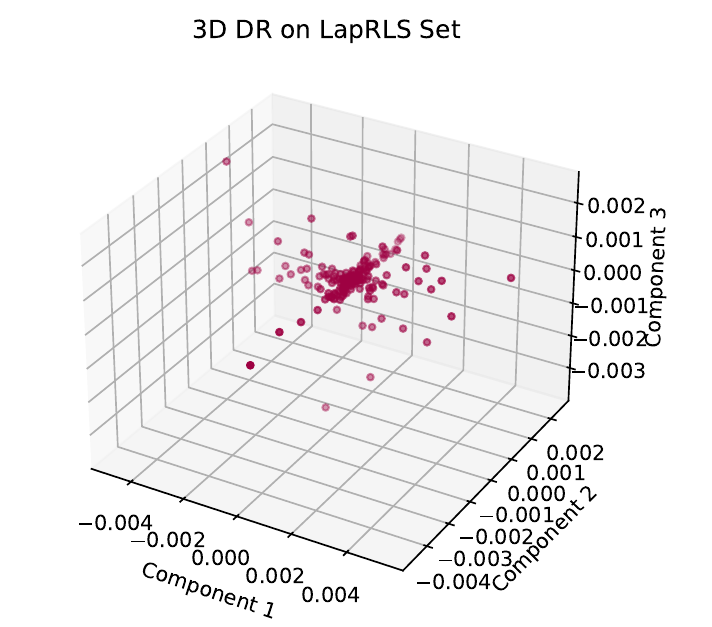}\label{fig:Lap}}
  \subfloat[LS]{\includegraphics[width=0.34\textwidth]{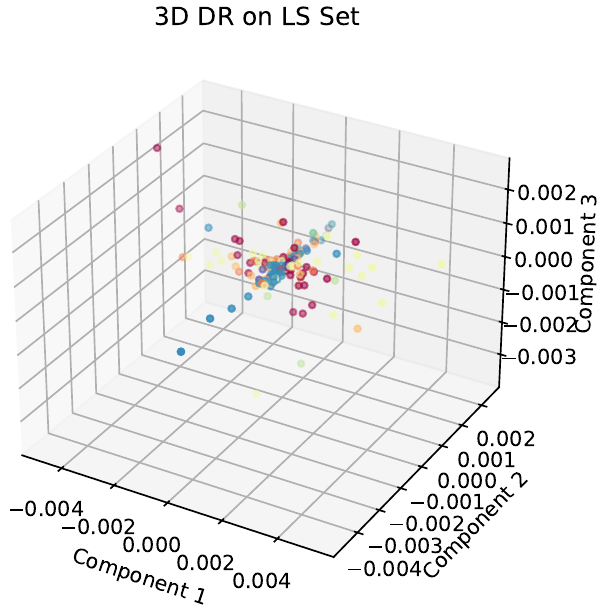}\label{fig:RL}}
  \caption{The color distribution on the reduced-dimensional test set based on the test labels is illustrated. In {\bf Fig\ref{fig:NHK},\ref{fig:Lap},\ref{fig:RL}}, the color distribution represents the output labels from NHK, Lap, and RL classifiers, respectively. In these figures, red represents the handwritten digit $0$, blue represents the handwritten digit $8$, and other colors indicate misclassified labels.}
  \label{Fig:DR}
\end{figure}
We can observe that the NHKRLS model classifier exhibits a relatively good classification performance with only a few data points misclassified. The LS classifier performs well in recognizing digit 8 but experiences numerous classification errors in identifying digit 0. In comparison, the Lap classifier incorrectly identifies all data points as 0, showing the poorest performance.

In conclusion, this section, through a comparison of the performance of NHKRLS, LapRLS, and LS models on various classification tasks in the MINIST dataset, demonstrates the classification accuracy of the NHKRLS model. It highlights the model's strong generalization ability with limited labeled data, confirming its superiority over the original model and traditional fully supervised classification models.
  \section{Conclusion}%
  In practical scenarios, acquiring a substantial amount of labeled data is often challenging and expensive. Manifold regularization model, as a form of semi-supervised learning, enhances model performance by effectively leveraging the positional relationships among unlabeled data, allowing models to achieve better results even when labeled data is scarce. However, the positional information of data may only reflect local features of the dataset in many cases, leading to sub-optimal performance of such methods.

This paper presents a feasible improvement framework for the manifold regularization model. It is based on a label propagation model constructed using an improved diffusion mapping algorithm  constructed by geodesic distances. By improving the model in this way, the classification task is viewed as a label diffusion process on the dataset, whereby finding a classifier is equivalent to finding the label distribution when label diffusion reaches a steady state. This improvement method combines heat conduction theory, heat kernel operators on manifolds, and Markov theory, providing new insights for the development of semi-supervised learning models.

There are several directions for future research:

\noindent{\bf  Selection of diffusion range and diffusion steps:}	 Different diffusion steps and the range of coloring samples in each diffusion step may affect the performance of the classifier. Due to varying sample densities around different samples, each diffusion step cannot ensure an equal number of labeled samples for both classes, which may result in one class having a significantly larger number of labeled samples than the other, thus affecting the classifier’s performance.\\
{\bf Performance on overly dense datasets:} The model performs well on datasets with clear decision boundaries. However, when the decision boundaries between two classes are unclear, unlabeled samples near the decision boundary will be influenced by labeled samples from both classes, thereby affecting the final classification result.\\
	{\bf Selection of diffusion models:} This paper constructs a label propagation model based on the classical heat conduction model. Different diffusion models are suitable for different types of datasets. Using alternative diffusion models to construct label propagation algorithms on datasets where certain heat conduction models perform poorly may result in better classifier performance. The discretization of infinitesimal generators for different diffusion processes can refer to the diffusion mapping algorithm proposed by Lafon

\bibliographystyle{unsrt}

\appendix

\section*{Appendix A. The asymptotic properties of improved probability transition operator}

In this appendix, we prove that the transition operator defined by geodesic distance has the same properties as it defined by Euclidean distance. Despite extensive research on the numerical computation of geodesic distances, there is still limited research on estimating geodesic distances globally. Fortunately, we don't need to rely on the global properties of geodesics. Locally, we can use orthogonal coordinate systems to calculate geodesic distance and obtain the transition operator defined by geodesics. We will provide proof of this in the subsequent part.

We assume that $\mathcal{M}$ is a smooth, compact sub-manifold of $\mathbb{R}^{n}$ with dimension $d$. For any point $p$ on the manifold $\mathcal{M}$, its tangent space is denoted as $T_{p}\mathcal{M}$. Assuming there exists a  orthonormal basis $\{v_{1},\cdots,v_{d}\}$ on the tangent space, any tangent vector $v$ can be reperesented in coordinates as $(u_{1},\cdots,u_{d})$. Additionally, using this orthonormal basis, we can generate local normal coordinates centered at $p$. If a point $q$ near $p$ has local coordinates $(s_{1},\cdots,s_{d})$, then $q=\exp_{p}(\sum_{i=1}^{d}{s_{i}v_{i}})$.

In this appendix, if $A$ is generated by Euclidean distance, it is denoted as $A_{E}$. If it is generated by geodesic distance, it is denoted as $A_{G}$. For kernel function $k_{\varepsilon}$, we denotes 
$$k_{E,\varepsilon}(x,y)=h(- \frac{\left\lVert x-y \right\rVert _{2}^2 }{\varepsilon})\quad k_{G,\varepsilon}(x,y) = h(-\frac{d_{G}(x,y)^2 }{\varepsilon})$$
, where $d_{G}(x,y)$ represents the geodesic distance on $\mathcal{M}$ between points $x$ and $y$.

\begin{lemma}
  $\forall x\in \mathcal{M}$, there exists an open neibourhood $U$ containing $x$, such that $\forall y\in U$, 
  $$d_{G}(x,y) = \sqrt{\sum_{i=1}^{d}s_{i}^2}$$
\end{lemma}
\begin{proof}
$\forall x\in\mathcal{M}$, the tangent space $T_{x}\mathcal{M}$ has a dimension of $d$. Hence, there exists an orthonormal basis consisting of $d$ vectors $\{v_{1},\cdots,v_{d}\}$. The parametrized geodesic curves generated by them are determined by(See \cite{doCarmo})
$$\left\{\begin{array}{l}
  \frac{D}{dt}(\gamma_{k}'(0))=\left(\sum_{i,j=1}^{n}\frac{dx_{j}}{dt}\frac{dx_{i}}{dt}\Gamma_{ij}^{k}+\frac{d^2 x_{i}}{dt}\right)v_{k} = 0\quad k=1,\cdots,d\\
  \gamma_{k}'(0) = e_{k}\\
  \gamma_{k}(0) = x
\end{array}\right.$$
, where $\Gamma_{ij}^{k}$ is the Christoffel symbol which is determined by the manifold itself. These equations are second-order nonlinear differential equations. According to Picard's theorem, for each equation, there exists an open set such that the solution of the equation exists and is unique. Taking the intersection of the preimage of these regions gives us the desired set $U$. Therefore, we obtain a normal coordinate system centered at $x$. For any $y\in U$, if $y$ has coordinates $(s_{1},\cdots,s_{d})$, $y = \exp_{x}(\sum_{i=1}^{d}s_{i}v_{i})$.  We know that the covariant derivative of the velocity $\gamma'(t)$ of a particle moving along a geodesic is zero. According to the geometric interpretation of the covariant derivative, the component of the particle's velocity in the tangent direction of the trajectory remains constant, which is equal to the initial velocity. From the form of the exponential map, it can be seen that the particle moves with $v=\sum_{i=1}^{d}s_{i}v_{i}$ as the initial velocity for unit time, and its arc length is
$$d_{G}(x,y) = \int_{0}^{1}|v|\,dt = \sqrt{\sum_{i=1}^{d}s_{i}^2} $$
\end{proof}
The following conclusions(See \cite{Lafon}) may be used in the subsequent theorems.
\begin{lemma}
  If $y\in\mathcal{M}$, is in a ball of radius $\varepsilon^\frac{1}{2}$ in $\mathbb{R}^{n}$, around $x$, then for $\varepsilon$ sufficiently small, there exists
  $$s_{i}=u_{i}+Q_{x,3}(u) +\mathcal{O}(\varepsilon^2)$$
  $$\left\lVert y-x \right\rVert _{2} ^2 = \left\lVert u \right\rVert ^2 + Q_{x,4}(u) + Q_{x,5}(u) + \mathcal{O}(\varepsilon^3) $$
  $$\det\left(\frac{dy}{du}\right) =1+Q_{x,2}(u) + Q_{x,3}(u) + \mathcal{O}(\varepsilon^2)$$
  where $Q_{x,m}(u)$ represents a homogeneous polynomial of de
\end{lemma}
We investigate some asymptotic properties of the linear functional defined by the kernel function generated by geodesic distance. Let
  $$G_{G,\varepsilon} = \int_{\mathcal{M}}k_{G,\varepsilon}(x,y)f(y)\,dy$$
  We prove that the estimation of $G_{G,\varepsilon}$ has same form with the it of $G_{\varepsilon}=\frac{1}{\varepsilon^{\frac{d}{2}}}\int_{\mathcal{M}}k_{E,\varepsilon}(x,y)f(y)\,dy$ That is,
\begin{theorem}
\label{theorem}
  Let $f\in C^3(\mathcal{M})$, $\gamma\in (0,\frac{1}{2})$. We have unifolrmly for all $x\in\mathcal{M_{\varepsilon^{\gamma}}}$
  $$G_{G,\varepsilon}f(x) = m_{0}f(x)+\varepsilon \frac{m_{2}}{2}(\omega(x)f(x)+\Delta f(x))+\mathcal{O}(\varepsilon^2)$$ 
  where $\mathcal{M}_{\varepsilon}=\{y\in\mathcal{M}|d_{E}(y,\p\mathcal{M})>\varepsilon\}$
\end{theorem}
This proof is similar to the euclidean version:
\begin{proof}
   Firstly, we give the estimation of $d_{G}(x,y)^2$. 
    Since $d_{G}(x,y)^2 = \sum_{i=1}^{d}{s_{i}^2}$ and according to the estimation of $s_{i}$, 
    $$d_{G}(x,y)^2 = \sum_{i=1}^{d}\left(u_{i}+Q_{x,3}(u)+\mathcal{O}(\varepsilon^2)\right)^2 = \left\lVert u \right\rVert _{2}^2 + Q_{x,4}(u) + Q_{x,6}(u)+\mathcal{O}(\varepsilon^4) $$
    Now we show that when $\varepsilon$ is sufficiently small, $G_{G,\varepsilon}f(x)$ is determined only by the points near $x$.
    $$G_{G,\varepsilon}f(x) = \frac{1}{\varepsilon^{\frac{d}{2}}}\left(\int_{B(x,{\varepsilon^{\gamma}})}k_{G,\varepsilon}(x,y)f(y)\,dy+\int_{\mathcal{M}\backslash B(x,{\varepsilon^{\gamma}})}k_{G,\varepsilon}(x,y)f(y)\,dy\right)$$
    Since $f$ is continous on $\mathcal{M}$, and $\mathcal{M}$ is compact, the maximum of $f$ exists. 
    \begin{align*}
      \frac{1}{\varepsilon^{\frac{d}{2}}}\int_{\mathcal{M}\backslash B(x,{\varepsilon^{\gamma}})}k_{G,\varepsilon}(x,y)f(y)\,dy&\leqslant \frac{1}{\varepsilon^{\frac{d}{2}}}\left\lVert f \right\rVert _{\infty} \int_{\mathcal{M}\backslash B(x,{\varepsilon^{\gamma}})}k_{G,\varepsilon}(x,y)\,dy\\
      &\leqslant \frac{1}{\varepsilon^{\frac{d}{2}}}\left\lVert f \right\rVert _{\infty} \int_{\mathcal{M}\backslash B(x,{\varepsilon^{\gamma}})}k_{E,\varepsilon}(x,y)\,dy\\
      &\leqslant\frac{1}{\varepsilon^{\frac{d}{2}}}\left\lVert f \right\rVert _{\infty} \int_{B(x,{\varepsilon^{\gamma}})^{c}}k_{E,\varepsilon}(x,y)\,dy
    \end{align*}
    Applying the method of integration by substitution to the above equation, consider the following mapping:\\
    $\map{\Phi}{\mathbb{R}^{n}}{\mathbb{R}^{n}}{t}{x+\sqrt{\varepsilon}t}$
    \begin{align*}
      &\frac{1}{\varepsilon^{\frac{d}{2}}}\left\lVert f \right\rVert _{\infty} \int_{B(x,{\varepsilon^{\gamma}})^{c}}k_{E,\varepsilon}(x,y)f(y)\,dy\\
      =& \frac{1}{\varepsilon^{\frac{d}{2}}}\left\lVert f \right\rVert _{\infty} \int_{B(0,{\varepsilon^{\gamma-\frac{1}{2}}})^{c}}h\left(\frac{\left\lVert x- (x+\sqrt{\varepsilon}t) \right\rVert _{2}^2}{\varepsilon} \right)\varepsilon^{\frac{n}{2}}\,dt\\
      =&\varepsilon^{\frac{n-d}{2}}\left\lVert f \right\rVert _{\infty} \int_{B(0,{\varepsilon^{\gamma-\frac{1}{2}}})^{c}}h\left( \left\lVert t \right\rVert _{2}^2 \right)\,dt
    \end{align*}
    Here we set the kernel function a Gaussian kernel. According to Coarea formula(See \cite{Evans}), we have
    \begin{align*}
      \int_{B(0,{\varepsilon^{\gamma-\frac{1}{2}}})^{c}}h\left( \left\lVert t \right\rVert _{2}^2 \right)\,dt&=\int_{0}^{+\infty}\int_{\mathbf{S}^{n-1}}r^{n-1} e^{-r^2}\mathbf{1}_{B(0,{\varepsilon^{\gamma-\frac{1}{2}}})^{c}}\,d\nu dr\\
      &=\int_{\varepsilon^{\gamma-\frac{1}{2}}}^{+\infty}\int_{\mathbf{S}^{n-1}} r^{n-1}e^{-r^2}\,d\nu dr\\
      &=|\mathbf{S}^{n-1}|\int_{\varepsilon^{\gamma-\frac{1}{2}}}^{+\infty}r^{n-1}e^{-r^2}\,dr
    \end{align*}
    Since the integrand decays exponentially, it can be proven to be an arbitrary high-order infinitesimal of $\varepsilon$ (simply by expanding $e^{r^2}$in a Taylor series to a sufficient number of terms, and observe that when $\gamma\in(0,\frac{1}{2})$, $\varepsilon^{\gamma-\frac{1}{2}}\to +\infty$, and then estimating the integral)
    
    Therefore, when epsilon is sufficiently small, the asymptotic properties of $G_{G,\varepsilon}f(x)$ are determined only by the points near $x$. Note that in proving the above proposition, we did not use normal coordinates. This indicates that the properties of G can be derived from local geodesic distances.
     Finally, we demonstrate the asymptotic estimation of $Gf(x)$ near $x$. Firstly, applying Taylor expansion to $k_{G,\varepsilon}$ at $\frac{\left\lVert u \right\rVert _{2}^2 }{\varepsilon}$
    $$k_{G}(x,y) = h\left(\frac{\left\lVert u \right\rVert _{2}^2 }{\varepsilon}\right) + \frac{1}{\varepsilon}h'\left(\frac{\left\lVert u \right\rVert _{2}^2 }{\varepsilon}\right)\left(Q_{x,4}(u)+Q_{x,6}(u)+\mathcal{O}(\varepsilon^{4})\right)+\mathcal{O}(\varepsilon^3)$$
    Then, in normal coordinates centered at $x$, the function $f$ on the manifold can be represented as a multivariate function $\tilde{f}(s_{1},\cdots,s_{d})$. Expanding the substituted function to the second order, we obtain:
    \begin{align*}
      \tilde{f}(s_{1},\cdots,s_{d}) &= \tilde{f}(0) + \sum_{i=1}^{d} \frac{\p \tilde{f}}{\p s_{i}}(0)s_{i}+\frac{1}{2}\sum_{i,j=1}^{d} \frac{\p^2 \tilde{f}}{\p s_{i}\p s_{j}}(0)s_{i}s_{j}+\sum_{i,j,k=1}^{d}\frac{\p^3 \tilde{f}}{\p s_{i}\p s_{j}\p s_{k}}(0)s_{i}s_{j}s_{k}\\
      &+\mathcal{O}(\varepsilon^2)\\
      &=\tilde{f}(0) + \sum_{i=1}^{d} \frac{\p \tilde{f}}{\p s_{i}}(0)u_{i}+\sum_{i,j=1}^{d} \frac{\p^2 \tilde{f}}{\p s_{i}\p s_{j}}(0)u_{i}u_{j}+Q_{x,3}(u)+\mathcal{O}(\varepsilon ^2)
    \end{align*}
    We multiply these estimates and use the conclusion from Lemma 4 to transform the integral over the sub-manifold to Euclidean space
    \begin{align*}
      G_{G,\varepsilon}f(x) =& \frac{1}{\varepsilon^{\frac{d}{2}}}\int_{B(0,\varepsilon^{\gamma})}\left(h\left(\frac{\left\lVert u \right\rVert _{2}^2 }{\varepsilon}\right) + \frac{1}{\varepsilon}h'\left(\frac{\left\lVert u \right\rVert _{2}^2 }{\varepsilon}\right)Q_{x,4}(u)+\mathcal{O}(\varepsilon^{2})\right)\\
      &\times\left(\tilde{f}(0) + \sum_{i=1}^{d} \frac{\p \tilde{f}}{\p s_{i}}(0)u_{i}+\frac{1}{2}\sum_{i,j=1}^{d} \frac{\p^2 \tilde{f}}{\p s_{i}\p s_{j}}(0)u_{i}u_{j}+Q_{x,3}(u)+\mathcal{O}(\varepsilon ^2)\right)\\
      &\times\left(1+Q_{x,2}(u) + Q_{x,3}(u) + \mathcal{O}(\varepsilon^2)\right)\,du
    \end{align*}
    The odd-power terms can be eliminated during integration. Expanding and simplifying the above expression, we get:
    \begin{align*}
      G_{G,\varepsilon}f(x) =\frac{1}{\varepsilon^{\frac{d}{2}}}\left(\tilde{f}(0)\int_{B(0,\varepsilon^{\gamma})}h\left(\frac{\left\lVert u \right\rVert _{2}^2 }{\varepsilon}\right)\,du+\tilde{f}(0)\frac{1}{\varepsilon}\int_{B(0,\varepsilon^{\gamma})}h'\left(\frac{\left\lVert u \right\rVert _{2}^2 }{\varepsilon}\right)Q_{x,4}(u)\right.\\
       \left(+ \tilde{f}(0)\int_{B(0,\varepsilon^{\gamma})}h\left(\frac{\left\lVert u \right\rVert _{2}^2 }{\varepsilon}\right)Q_{x,2}(u)\,du+ \frac{1}{2}\sum_{i=1}^{d}{\frac{\p \tilde{f}}{\p s_{i}^2}(0)}\int_{B(0,\gamma)}u_{1}^2 h\left(\frac{\left\lVert u \right\rVert _{2}^2 }{\varepsilon}\right)\,du\right)+\mathcal{O}(\varepsilon^2)
    \end{align*}
    Similarly, under a scaling transformation, we have:
    \begin{align*}
      G_{G,\varepsilon}f(x) =& \tilde{f}(0)\int_{B(0,\varepsilon^{\gamma-\frac{1}{2}})}h(\left\lVert u \right\rVert _{2}^2 )\,du + \frac{\varepsilon}{2}\int_{B(0,\varepsilon^{\gamma-\frac{1}{2}})}\left(\sum_{i=1}^{d}\frac{\p^2 \tilde{f}}{\p s_{i}^2}(0)\right)\int_{B(0,\varepsilon^{\gamma-\frac{1}{2}})}u_{1}^2 h(\left\lVert u \right\rVert _{2}^2 )\,du\\
      &+\varepsilon \tilde{f}(0)\int_{B(0,\varepsilon^{\gamma-\frac{1}{2}})}\left(Q_{x,4}(u)h'(\left\lVert u \right\rVert _{2}^2 )+Q_{x,2}(u)h(\left\lVert u \right\rVert _{2}^2 )\right)\,du +\mathcal{O}(\varepsilon^2 )
    \end{align*}
    Let $m_{0} = \int_{\mathbb{R}^{d}}h(\left\lVert u \right\rVert _{2}^2 )\,du$, $m_{2} = \int_{\mathbb{R}^{d}}u_{1}^2 h(\left\lVert u \right\rVert _{2}^2 )\,du$, and
    \\$\omega(x) =\frac{2}{m_{2}}\int_{\mathbb{R}^{d}}\left(Q_{x,4}(u)h'(\left\lVert u \right\rVert _{2}^2 )+Q_{x,2}(u)h(\left\lVert u \right\rVert _{2}^2 )\right)\,du $.
    
    Let $\varepsilon\to 0$, we finally obtain
    $$G_{G,\varepsilon}f(x) =m_{0}f(x)+\varepsilon \frac{m_{2}}{2}(\omega(x)f(x)+\Delta f(x))+\mathcal{O}(\varepsilon^2)$$
    which is same with the euclidean case.
\end{proof}
Using the above result, we can deduce the following conclusion
\begin{proposition}
  $\lim_{\varepsilon \to 0}P_{G,\varepsilon}^{\frac{t}{\varepsilon}}=e^{t\Delta}$
\end{proposition}
The proof is same with euclidean case, so, I won't go into further detail.

\section*{Appendix B. Some properties of extended label propogation function}
In this appendix, we will prove some properties of the extended label propagation function mentioned in constructing the NHK Regularization model. 
Firstly, we prove the existence of an extended label propagation function
\begin{theorem}[The existence of Label Propagation function]
  $X$ is a dataset with a finite number of elements. Considering the binary classification problem, $X_{l}$ is the labeled dataset with corresponding label set $Y_{l}$, and $X_{u}$ is the unlabeled dataset. There exists $Q(x)\in C^{\infty}(\mathcal{M})$ such that
  $$\mathrm{range}(Q|_{X_{l}}) = \{+1,-1\},\quad \mathrm{range}(Q|_{X_{u}})=\{0\}$$  
  where $\mathcal{M}\subset \mathbb{R}^{n}$ is a $d$-dimensional sub-manifold the dataset embedded.
\end{theorem}
\begin{proof}
  Our proof is divided into two steps\\
    {\bf 1. Find disjoint regions} Since $X_{l}\cap X_{u}=\varnothing$, and $X_{u}$ is consisted of finite number of point in $\mathcal{M}$, we have $X_{l}\subset  X_{u}^{c}$, where $X_{u}^{c}$ is open. Since every single point is compact in finite Euclidean space, which is the ambient space of sub-manifold $\mathcal{M}$, for each $x\in X_{l}$, we can find an open set $V_{x}$ on $\mathcal{M}$ such that(See \cite{rudin1974real})
    $$\{x\}\subset V_{x}\subset\overline{V_{x}}\subset X_{u}^{c}\cap\left(\bigcup_{y\in X_{l},y\neq x}\{y\}^{c}\right)$$
    Here, $A^{c}=\mathcal{M}\backslash A$. Do the same process for each $x\in X_{l}$, then we obtain a collection of disjoint open set $\{V_{x}\}_{x\in X_{l}}$, whose elements have compact closure.
    \begin{figure}[H]
      \centering

\tikzset{every picture/.style={line width=0.75pt}} 

\begin{tikzpicture}[x=0.75pt,y=0.75pt,yscale=-1,xscale=1]

\draw  [fill={rgb, 255:red, 155; green, 155; blue, 155 }  ,fill opacity=0.55 ] (134.14,125.94) .. controls (164.14,126.94) and (211,103) .. (228,128) .. controls (245,153) and (285.14,200.94) .. (270.14,233.94) .. controls (255.14,266.94) and (146.14,288.94) .. (124.14,278.94) .. controls (102.14,268.94) and (84.14,232.94) .. (83.14,202.94) .. controls (82.14,172.94) and (104.14,124.94) .. (134.14,125.94) -- cycle ;
\draw    (167.14,178.94) ;
\draw [shift={(167.14,178.94)}, rotate = 0] [color={rgb, 255:red, 0; green, 0; blue, 0 }  ][fill={rgb, 255:red, 0; green, 0; blue, 0 }  ][line width=0.75]      (0, 0) circle [x radius= 3.35, y radius= 3.35]   ;
\draw    (200.14,203.94) ;
\draw [shift={(200.14,203.94)}, rotate = 0] [color={rgb, 255:red, 0; green, 0; blue, 0 }  ][fill={rgb, 255:red, 0; green, 0; blue, 0 }  ][line width=0.75]      (0, 0) circle [x radius= 3.35, y radius= 3.35]   ;
\draw    (158.14,213.94) ;
\draw [shift={(158.14,213.94)}, rotate = 0] [color={rgb, 255:red, 0; green, 0; blue, 0 }  ][fill={rgb, 255:red, 0; green, 0; blue, 0 }  ][line width=0.75]      (0, 0) circle [x radius= 3.35, y radius= 3.35]   ;
\draw [color={rgb, 255:red, 255; green, 255; blue, 255 }  ,draw opacity=1 ]   (208,244) ;
\draw [shift={(208,244)}, rotate = 0] [color={rgb, 255:red, 255; green, 255; blue, 255 }  ,draw opacity=1 ][fill={rgb, 255:red, 255; green, 255; blue, 255 }  ,fill opacity=1 ][line width=0.75]      (0, 0) circle [x radius= 3.35, y radius= 3.35]   ;
\draw [color={rgb, 255:red, 255; green, 255; blue, 255 }  ,draw opacity=1 ]   (106.14,236.94) ;
\draw [shift={(106.14,236.94)}, rotate = 0] [color={rgb, 255:red, 255; green, 255; blue, 255 }  ,draw opacity=1 ][fill={rgb, 255:red, 255; green, 255; blue, 255 }  ,fill opacity=1 ][line width=0.75]      (0, 0) circle [x radius= 3.35, y radius= 3.35]   ;
\draw [color={rgb, 255:red, 255; green, 255; blue, 255 }  ,draw opacity=1 ]   (175.14,202.94) ;
\draw [shift={(175.14,202.94)}, rotate = 0] [color={rgb, 255:red, 255; green, 255; blue, 255 }  ,draw opacity=1 ][fill={rgb, 255:red, 255; green, 255; blue, 255 }  ,fill opacity=1 ][line width=0.75]      (0, 0) circle [x radius= 3.35, y radius= 3.35]   ;
\draw [color={rgb, 255:red, 255; green, 255; blue, 255 }  ,draw opacity=1 ]   (164.14,242.94) ;
\draw [shift={(164.14,242.94)}, rotate = 0] [color={rgb, 255:red, 255; green, 255; blue, 255 }  ,draw opacity=1 ][fill={rgb, 255:red, 255; green, 255; blue, 255 }  ,fill opacity=1 ][line width=0.75]      (0, 0) circle [x radius= 3.35, y radius= 3.35]   ;
\draw [color={rgb, 255:red, 255; green, 255; blue, 255 }  ,draw opacity=1 ]   (226,166) ;
\draw [shift={(226,166)}, rotate = 0] [color={rgb, 255:red, 255; green, 255; blue, 255 }  ,draw opacity=1 ][fill={rgb, 255:red, 255; green, 255; blue, 255 }  ,fill opacity=1 ][line width=0.75]      (0, 0) circle [x radius= 3.35, y radius= 3.35]   ;
\draw [color={rgb, 255:red, 255; green, 255; blue, 255 }  ,draw opacity=1 ]   (136.14,155.94) ;
\draw [shift={(136.14,155.94)}, rotate = 0] [color={rgb, 255:red, 255; green, 255; blue, 255 }  ,draw opacity=1 ][fill={rgb, 255:red, 255; green, 255; blue, 255 }  ,fill opacity=1 ][line width=0.75]      (0, 0) circle [x radius= 3.35, y radius= 3.35]   ;
\draw [color={rgb, 255:red, 255; green, 255; blue, 255 }  ,draw opacity=1 ]   (183.14,130.94) ;
\draw [shift={(183.14,130.94)}, rotate = 0] [color={rgb, 255:red, 255; green, 255; blue, 255 }  ,draw opacity=1 ][fill={rgb, 255:red, 255; green, 255; blue, 255 }  ,fill opacity=1 ][line width=0.75]      (0, 0) circle [x radius= 3.35, y radius= 3.35]   ;
\draw [color={rgb, 255:red, 255; green, 255; blue, 255 }  ,draw opacity=1 ]   (251.14,205.94) ;
\draw [shift={(251.14,205.94)}, rotate = 0] [color={rgb, 255:red, 255; green, 255; blue, 255 }  ,draw opacity=1 ][fill={rgb, 255:red, 255; green, 255; blue, 255 }  ,fill opacity=1 ][line width=0.75]      (0, 0) circle [x radius= 3.35, y radius= 3.35]   ;
\draw  [fill={rgb, 255:red, 155; green, 155; blue, 155 }  ,fill opacity=0.55 ] (414.14,122) .. controls (444.14,123) and (491,99.06) .. (508,124.06) .. controls (525,149.06) and (565.14,197) .. (550.14,230) .. controls (535.14,263) and (426.14,285) .. (404.14,275) .. controls (382.14,265) and (364.14,229) .. (363.14,199) .. controls (362.14,169) and (384.14,121) .. (414.14,122) -- cycle ;
\draw    (454.14,177.94) ;
\draw [shift={(454.14,177.94)}, rotate = 0] [color={rgb, 255:red, 0; green, 0; blue, 0 }  ][fill={rgb, 255:red, 0; green, 0; blue, 0 }  ][line width=0.75]      (0, 0) circle [x radius= 3.35, y radius= 3.35]   ;
\draw    (487.14,202.94) ;
\draw [shift={(487.14,202.94)}, rotate = 0] [color={rgb, 255:red, 0; green, 0; blue, 0 }  ][fill={rgb, 255:red, 0; green, 0; blue, 0 }  ][line width=0.75]      (0, 0) circle [x radius= 3.35, y radius= 3.35]   ;
\draw    (445.14,212.94) ;
\draw [shift={(445.14,212.94)}, rotate = 0] [color={rgb, 255:red, 0; green, 0; blue, 0 }  ][fill={rgb, 255:red, 0; green, 0; blue, 0 }  ][line width=0.75]      (0, 0) circle [x radius= 3.35, y radius= 3.35]   ;
\draw [color={rgb, 255:red, 255; green, 255; blue, 255 }  ,draw opacity=1 ]   (495,243) ;
\draw [shift={(495,243)}, rotate = 0] [color={rgb, 255:red, 255; green, 255; blue, 255 }  ,draw opacity=1 ][fill={rgb, 255:red, 255; green, 255; blue, 255 }  ,fill opacity=1 ][line width=0.75]      (0, 0) circle [x radius= 3.35, y radius= 3.35]   ;
\draw [color={rgb, 255:red, 255; green, 255; blue, 255 }  ,draw opacity=1 ]   (393.14,235.94) ;
\draw [shift={(393.14,235.94)}, rotate = 0] [color={rgb, 255:red, 255; green, 255; blue, 255 }  ,draw opacity=1 ][fill={rgb, 255:red, 255; green, 255; blue, 255 }  ,fill opacity=1 ][line width=0.75]      (0, 0) circle [x radius= 3.35, y radius= 3.35]   ;

\draw [color={rgb, 255:red, 255; green, 255; blue, 255 }  ,draw opacity=1 ]   (462.14,201.94) ;
\draw [shift={(462.14,201.94)}, rotate = 0] [color={rgb, 255:red, 255; green, 255; blue, 255 }  ,draw opacity=1 ][fill={rgb, 255:red, 255; green, 255; blue, 255 }  ,fill opacity=1 ][line width=0.75]      (0, 0) circle [x radius= 3.35, y radius= 3.35]   ;

\draw [color={rgb, 255:red, 255; green, 255; blue, 255 }  ,draw opacity=1 ]   (451.14,241.94) ;
\draw [shift={(451.14,241.94)}, rotate = 0] [color={rgb, 255:red, 255; green, 255; blue, 255 }  ,draw opacity=1 ][fill={rgb, 255:red, 255; green, 255; blue, 255 }  ,fill opacity=1 ][line width=0.75]      (0, 0) circle [x radius= 3.35, y radius= 3.35]   ;

\draw [color={rgb, 255:red, 255; green, 255; blue, 255 }  ,draw opacity=1 ]   (513,165) ;
\draw [shift={(513,165)}, rotate = 0] [color={rgb, 255:red, 255; green, 255; blue, 255 }  ,draw opacity=1 ][fill={rgb, 255:red, 255; green, 255; blue, 255 }  ,fill opacity=1 ][line width=0.75]      (0, 0) circle [x radius= 3.35, y radius= 3.35]   ;

\draw [color={rgb, 255:red, 255; green, 255; blue, 255 }  ,draw opacity=1 ]   (423.14,154.94) ;
\draw [shift={(423.14,154.94)}, rotate = 0] [color={rgb, 255:red, 255; green, 255; blue, 255 }  ,draw opacity=1 ][fill={rgb, 255:red, 255; green, 255; blue, 255 }  ,fill opacity=1 ][line width=0.75]      (0, 0) circle [x radius= 3.35, y radius= 3.35]   ;

\draw [color={rgb, 255:red, 255; green, 255; blue, 255 }  ,draw opacity=1 ]   (469.14,129.94) ;
\draw [shift={(469.14,129.94)}, rotate = 0] [color={rgb, 255:red, 255; green, 255; blue, 255 }  ,draw opacity=1 ][fill={rgb, 255:red, 255; green, 255; blue, 255 }  ,fill opacity=1 ][line width=0.75]      (0, 0) circle [x radius= 3.35, y radius= 3.35]   ;

\draw [color={rgb, 255:red, 255; green, 255; blue, 255 }  ,draw opacity=1 ]   (538.14,204.94) ;
\draw [shift={(538.14,204.94)}, rotate = 0] [color={rgb, 255:red, 255; green, 255; blue, 255 }  ,draw opacity=1 ][fill={rgb, 255:red, 255; green, 255; blue, 255 }  ,fill opacity=1 ][line width=0.75]      (0, 0) circle [x radius= 3.35, y radius= 3.35]   ;
 
\draw   (439.14,177.94) .. controls (439.14,169.66) and (445.86,162.94) .. (454.14,162.94) .. controls (462.42,162.94) and (469.14,169.66) .. (469.14,177.94) .. controls (469.14,186.23) and (462.42,192.94) .. (454.14,192.94) .. controls (445.86,192.94) and (439.14,186.23) .. (439.14,177.94) -- cycle ;

\draw   (475.14,202.94) .. controls (475.14,196.32) and (480.51,190.94) .. (487.14,190.94) .. controls (493.77,190.94) and (499.14,196.32) .. (499.14,202.94) .. controls (499.14,209.57) and (493.77,214.94) .. (487.14,214.94) .. controls (480.51,214.94) and (475.14,209.57) .. (475.14,202.94) -- cycle ;

\draw   (433.14,212.94) .. controls (433.14,206.32) and (438.51,200.94) .. (445.14,200.94) .. controls (451.77,200.94) and (457.14,206.32) .. (457.14,212.94) .. controls (457.14,219.57) and (451.77,224.94) .. (445.14,224.94) .. controls (438.51,224.94) and (433.14,219.57) .. (433.14,212.94) -- cycle ;

\draw (128,256) node [anchor=north west][inner sep=0.75pt]   [align=left] {$\displaystyle \mathcal{M}$};

\draw (410,252) node [anchor=north west][inner sep=0.75pt]   [align=left] {$\displaystyle \mathcal{M}$};

\draw (469,160) node [anchor=north west][inner sep=0.75pt]   [align=left] {$\displaystyle V_{x}$};

\end{tikzpicture}
      \caption{The main idea of Proof:White-colored circle represents the point in $X_{u}$, it is removed from $\mathcal{M}$. The black-colored circle is in $X_{l}$. We can find a collection of disjoint closed sets containing them.}
      \label{}
    \end{figure}
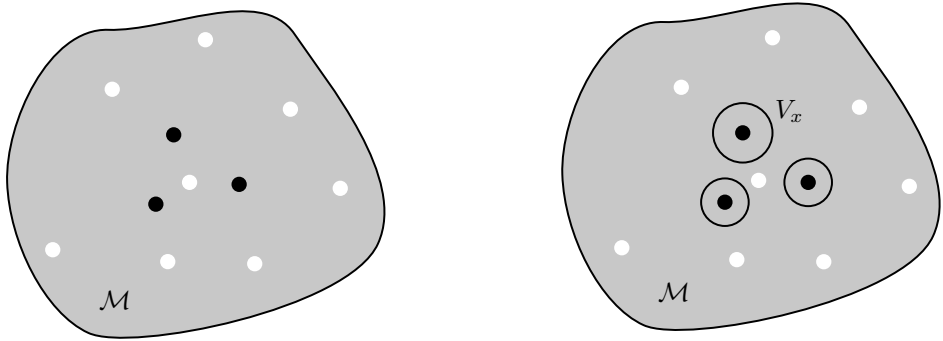
    \noindent {\bf 2. Construction of smooth extended function:} Now construct the extended function. According to the definition of sub-manifold, there exists an atlas $\{(U_{\alpha},\Phi_{\alpha})\}_{\alpha\in\Lambda}$,  for each $x\in\mathcal{M}$, we can find $U_{\alpha_{0}}$ cover $x$ and open set $V\subset \mathbb{R}^{n}$, such that $\{U_{\alpha}\}_{\alpha\in\Lambda}$ is an open cover of $\mathcal{M}$ and $\{\Phi_{\alpha}\}_{\alpha\in\Lambda}$ is a collection of local diffeomorphism. Choose an arbitrary $V_{x_{0}}\in\{V_{x}\}_{x\in X_{l}}$, then we can find finite many ${U_\alpha}$ cover $\overline{V}_{x_{0}}$, denoted by $\{U_i\}_{i=1}^{N}$.
    Choose $U_{i_{0}}$ contains $x_{0}$, after translation, we can find open set $W$ and $\Phi_{i_{0}}$ such that
    $$\Phi_{i_{0}}: U_{i_{0}}\cap V_{x_{0}}\to (\mathbb{R}^{d}\times\{0\}^{n-d})\cap W$$
    is a $C^{\infty}$ diffeomorphism, and $\Phi(x_{0})=0$.
    Considering $$f(x) =\left\{\begin{array}{ll}
      e^{\frac{1}{|x|^2-1}+1}& x\in(-1,1)\\
      0                      & \mbox{otherwise}
    \end{array}\right.$$, we can prove that it is $C^{\infty}$. Since $U_{i_{0}}\cap V_{x_{0}}$ is open set, there exists $r_{0}>0$ such that the pre-image of open ball $\Phi^{-1}(B(0,r_{0})\cap (\mathbb{R}^{d}\times\{0\}^{n-d}))$ contained by this open set. Let $\psi(x) = f(\frac{\left\lVert x \right\rVert _{2} }{r_{0}} )$, $\varphi(x) = \Phi^{*}\psi(x)=(\psi\circ\Phi)(x)$, then $\varphi$ is $C^{\infty}$ function in $U_{i_{0}}\cap V_{x_{0}}$, and $\mathrm{supp}\varphi\subset\subset U_{i_{0}}\cap V_{x_{0}}$. Finally, let $$\rho(x) =\left\{\begin{array}{ll}
      y_{0}\varphi(x)& x\in U_{i_{0}}\cap V_{x_{0}}\\
      0         & x\in \mathcal{M}\backslash (U_{i_{0}}\cap V_{x_{0}})
    \end{array}\right.$$
    where $y_{0}$ is the label of $x_{0}$. It can be seen that $\rho$ is a smooth function on $\mathcal{M}$, and $\rho(x_{0})=y_{0}$. Do same process for each $V_{x}$, we obtain a collection of smooth function $\{\rho_{x}\}_{x\in X_{l}}$, and for each $x\in X_{l}$, $\rho_{x}$ compact support is contained by $V_{x}$. Therefore, we desired extended label propagation function can be expressed as 
    $$\bar{u} = \sum_{x\in X_{l}}\rho_{x}$$
\end{proof}
The extended label propagation function follows the heat conduction equation. Since the function values remain unchanged on the labeled dataset, these data points can be considered constant-temperature heat sources. Therefore, the extended label propagation follows the following diffusion equation(See \cite{thermal}):
\begin{gather}
  \left\{\begin{array}{l}
    \p_{t}u(x,t) = \Delta u(x,t) + Q(x) \quad (x,t)\in\mathcal{M}\times \mathbb{R}_{+}\\
    u(x,0) = Q(x)\quad x\in\mathcal{M}
  \end{array}\right.
\end{gather}
where $Q(x)$ is the label distribution of labeled data points, and for convenience, we denote $u$ as the extended label propagation function in this appendix. We first prove that after a sufficiently long time, the distribution of labels on the dataset will reach a steady state. 
\begin{theorem}
  $\mathcal{M}$ is a $d$ dimensional compact sub-manifold in $\mathbb{R}^{n}$, when $\mathrm{Ric}(\mathcal{M})\geqslant 0$, that is, for any $w\in T\mathcal{M}$,
  $\mathrm{Ric}(w,w)\geqslant 0$. Then there exists function $f(x)$ such that 
  $$\lim_{t \to \infty}u(x,t) = f(x)$$
\end{theorem}
\begin{proof}
  Firstly, we perform a variable substitution on the equation. Since the Poisson equation on a compact manifold must have a fundamental solution $G(x,y)$(See \cite{Audin}), let $w$ satisfy the equation: $w(x) = \int_{\mathcal{M}}G(x,y)Q(y)\,dy$. Then we can transform the equation(.1) into
$$\p_{t}(u(x,t)+w(x)) =\Delta (u(x,t)+w(x))$$
For convenience in discussion, let $v(x,t)=u(x,t)+w(x)$, then we express the equation in operator form, 
$$v(x,t)= e^{t\Delta}v(x,0)=\int_{\mathcal{M}}p_{t}(x,y)v(y,0)\,dy$$
where $p_{t}(x,y)$ is the fundamental solution of the heat diffusion equation on the manifold. Although we cannot find an explicit solution to the heat equation on the manifold, intuitively, if there are no constant-temperature heat sources in space, after a sufficiently long time, the heat in space will dissipate completely. Therefore, we speculate that the solution to the heat equation should contain a decaying term. In fact, according to the work of Yau and Li in 1986(See \cite{li1986parabolic}), a fundamental solution to the heat equation on a complete Riemann manifold with $\mathrm{Ric}(\mathcal{M})\geqslant -K^2$ can be controlled by Gaussian functions, i.e., it has an upper bound of the following form:
$$p_{t}(x,y) \leqslant  \frac{C(\delta,d)}{|B(x,\sqrt{t})|^{\frac{1}{2}}|B(y,\sqrt{t})|^{\frac{1}{2}}}e^{-\frac{d_{G}(x,y)^2}{(4+\delta)t}} e^{C_{1}(d)\delta K^2 t}$$
where $C(\delta,d)$, $C_{1}(d)$, $\delta$ are constants determined by manifold itself. In particular, when $K=0$, we can prove that
$$p_{t}(x,y)\leqslant C \frac{e^{-\frac{d_{G}(x,y)^2}{ct}}}{|B(x,\sqrt{t})|}$$ 
In fact, according to Bishop comparison theorem(See \cite{doCarmo}), 
$$\frac{|B(x,r)|}{|B(x,s)|}\leqslant \left(\frac{r}{s}\right)^{d}$$
for any $r\geqslant s$, we have following simple estimation:
$$\frac{|B(x,\sqrt{t})|}{|B(y,\sqrt{t})|}\leqslant \frac{|B(y,\sqrt{t}+d_{G}(x,y))|}{|B(y,\sqrt{t})|}\leqslant\left(\frac{\sqrt{t}+d(x,y)}{\sqrt{t}}\right)^{d}$$
Since there exists constant $C$ such that $\ln(1+m)< Cm^2$, let $m=\frac{d_{G}(x,y)}{\sqrt{t}}$, then we can find a constant $c$ such that
$$\left(1+\frac{d(x,y)}{\sqrt{t}}\right)^{\frac{n}{2}}\leqslant C e^{C' \frac{d_{G}(x,y)^2}{t}}$$
i.e. $$C(\delta,n) e^{- \frac{d_{G}(x,y)^2}{(4+\delta)t}}\left(1+\frac{d(x,y)}{\sqrt{t}}\right)^{\frac{n}{2}}\leqslant C e^{-\frac{d_{G}(x,y)^2}{ct}}$$
Therefore, there exists positive constants $C$, $c$ such that
$$p_{t}(x,y)\leqslant \frac{C(\delta,n)}{|B(x,\sqrt{t})|} \left(1+\frac{d(x,y)}{\sqrt{t}}\right)^{\frac{d}{2}}e^{- \frac{d_{G}(x,y)^2}{(4+\delta)t}} \leqslant\frac{C}{|B(x,\sqrt{t})|} e^{-\frac{d_{G}(x,y)^2}{ct}}$$
Since $p_{t}$ is positive, according to dominant convergent theorem, we have
$$\lim_{t \to \infty}v(x,t)= \lim_{t \to \infty}\int_{\mathcal{M}}p_{t}(x,y)v(y,0)\,dy\leqslant C\int_{\mathcal{M}}\lim_{t \to \infty}\frac{e^{-\frac{d_{G}(x,y)^2}{ct}}}{t^{\frac{d}{2}}}v(y,0)\,dy=0$$
This proves that our conjecture is correct: when the time is sufficiently long, the heat distribution on the manifold will reach a steady state, i.e.
$$\lim_{t \to \infty}u(x,t) = -w(x)$$
where $\Delta w(x) =Q(x)$
\end{proof}

\end{document}